\documentclass{article}

    \PassOptionsToPackage{numbers, compress}{natbib}
 \usepackage[preprint]{neurips_2026}

\usepackage[utf8]{inputenc} % allow utf-8 input
\usepackage[T1]{fontenc}    % use 8-bit T1 fonts
\usepackage{hyperref}       % hyperlinks
\usepackage{url}            % simple URL typesetting
\usepackage{booktabs}       % professional-quality tables
\usepackage{amsfonts}       % blackboard math symbols
\usepackage{nicefrac}       % compact symbols for 1/2, etc.
\usepackage{microtype}      % microtypography
\usepackage{xcolor}         % colors
\usepackage{amsthm}         %% Theorems
\usepackage{graphicx}       %% images
\usepackage{multirow}       %% Multirow in tables
\usepackage{adjustbox}      %% Fitting tables
\usepackage{wrapfig}

\usepackage{xspace}
\usepackage[dvipsnames]{xcolor}
\usepackage{graphicx}
\usepackage{amsmath,amssymb,amsfonts}
\usepackage{todonotes}
\usepackage{outlines}
\usepackage{mathrsfs}  % For \mathscr
\usepackage{euscript}  % For \EuScript

\newcommand{\eq}[1]{
    \begin{align}
        #1
    \end{align}
}

\definecolor{tdbg}{HTML}{FFCC00}
\definecolor{fixbg}{HTML}{CC0000}
\definecolor{revbg}{HTML}{EB0CB7}
\definecolor{checkbg}{HTML}{95FA2F}
\definecolor{changebg}{HTML}{12C915}
\definecolor{ongoingbg}{HTML}{19FFA5}
\definecolor{ablatebg}{HTML}{1F3FF5}
\definecolor{modbg}{HTML}{CC00CC}
\definecolor{tdlbg}{HTML}{BFFFF0}
\definecolor{tdnbg}{HTML}{FFE3BF}
\definecolor{notebg}{HTML}{FF0000}
\definecolor{donebg}{HTML}{00C3FF}
\definecolor{textdark}{HTML}{1F1F1F}
\definecolor{darkgreen}{RGB}{0,120,60}

\newtheorem{assumption}{Assumption}
\newtheorem{proposition}{Proposition}
\newtheorem{lemma}{Lemma}
\newtheorem{theorem}{Theorem}
\newtheorem{corollary}{Corollary}

\title{\textbf{Towards Metric-Faithful Neural Graph Matching}}

\author{
    Jyotirmaya Shivottam$^{1,2}$ \\
    \texttt{\url{jyotirmaya.shivottam@niser.ac.in}} \\
    $^1$National Institute of Science Education and Research (NISER), Bhubaneswar, India \\
    $^2$Homi Bhabha National Institute, Mumbai, India
    \And
    Subhankar Mishra$^{1,2}$ \\
    \texttt{\url{smishra@niser.ac.in}} \\
    $^1$National Institute of Science Education and Research (NISER), Bhubaneswar, India \\
    $^2$Homi Bhabha National Institute, Mumbai, India
}
\begin{document}

\maketitle

\begin{abstract}
Graph Edit Distance (GED) is a fundamental, albeit NP-hard, metric for structural graph similarity. Recent neural graph matching architectures approximate GED by first encoding graphs with a Graph Neural Network (GNN) and then applying either a graph-level regression head or a matching-based alignment module. Despite substantial architectural progress, the role of encoder geometry in neural GED estimation remains poorly understood. In this paper, we develop a theoretical framework that connects encoder geometry to GED estimation quality for two broad classes of neural GED estimators: graph similarity predictors and alignment-based methods. On fixed graph collections, where the doubly-stochastic metric $d_{\mathrm{DS}}$ is comparable to GED, we show that graph-level bi-Lipschitz encoders yield controlled GED surrogates and improved ranking stability; for matching-based estimators, node-level bi-Lipschitz geometry propagates to encoder-induced alignment costs and the resulting optimized alignment objective. We instantiate this perspective using FSW-GNN~\citep{sverdlov2025fsw}, a bi-Lipschitz WL-equivalent encoder, as a drop-in replacement in representative neural GED architectures. Across representative baselines and benchmark datasets, the resulting geometry-aware variants significantly improve GED prediction and ranking metrics. A faithfulness case study of untrained encoders, together with ablations and transfer experiments, supports the view that these gains arise from improved representation geometry, positioning encoder geometry as a useful design principle for neural graph matching.
\end{abstract}

%%%%%%%%%%%%%%%%%%%%%%%%%%%%%%%%%%%%
%% Intro
%%%%%%%%%%%%%%%%%%%%%%%%%%%%%%%%%%%%

\section{Introduction}
\label{sec:intro}

Graph matching and graph similarity are central to applications such as molecular retrieval, program analysis, and structured search~\citep{wang2025neural,xu2025graph,liu2024graph}. A standard target in this setting is Graph Edit Distance (GED), the minimum-cost sequence of node and edge edits required to transform one graph into another~\citep{riesen}. GED is mathematically well founded but computationally difficult: exact computation reduces to a hard combinatorial matching problem and quickly becomes impractical as graph size grows~\citep{zheng2025grasp}. This difficulty has long motivated continuous relaxations of graph matching. More broadly, the classical literature already shows that the quality of a relaxation can strongly affect faithfulness to the target correspondence: Lyzinski et al.~show that, in a correlated random graph setting, a common convex relaxation can ``almost always yield[] the wrong matching''~\citep{lyzinski2014graph}. Bernard et al.~show that one can seek tighter \emph{lifting-free} convex relaxations that remain scalable, underscoring that the geometry and tightness of continuous surrogates matter even before learning enters the pipeline~\citep{bernard2018ds}.

To avoid the full combinatorial cost of exact GED, a large literature studies neural GED estimators. These methods typically fall into two broad families: \textit{graph-similarity predictors} (e.g., GMN~\citep{li2019graph}, SimGNN~\citep{bai2020simgnn}, GraSP~\citep{zheng2025grasp}), which map each graph to an embedding and regress a scalar distance, and \textit{matching-based estimators} (e.g., GEDGNN~\citep{piao2023computing}, GraphEDX~\citep{jain2024graph}, FGWAlign~\citep{tang2025fused}), which build and score a relaxed cross-graph alignment object. Despite their differences, most such pipelines share the same front end: a GNN encoder produces graph-level or node-level representations, followed by a downstream prediction or matching module~\citep{xu2025graph,liu2024graph}.

Yet the geometry induced by this encoder is rarely analyzed explicitly. Existing methods typically motivate encoder choice through expressive power or architectural convenience~\citep{xu2019powerful,piao2023computing}, while the relationship between representational distortion and GED estimation quality remains largely unexplored. If the encoder substantially distorts graph-level distances or cross-graph node discrepancies, the downstream module must operate on a poorly conditioned surrogate of the underlying edit geometry. This concern is especially relevant for Message Passing Neural Networks (MPNNs): while their discriminative power is bounded by the \(1\)-Weisfeiler--Lehman test~\citep{xu2019powerful}, recent theory shows that standard sum-based architectures do not enjoy strong lower-Lipschitz guarantees, and their representation separation quality can deteriorate rapidly with depth~\citep{davidsonHolderStabilityMultiset2024}. Recent surveys and benchmark studies further highlight substantial generalization failures in current neural GED systems~\citep{xu2025graph,roy2025position}.

\begin{wrapfigure}{r}{0.65\textwidth}
    \vspace{-10pt}
    \centering
    \includegraphics[trim=0.4cm 0.4cm 0.4cm 0.4cm, clip, width=\linewidth]{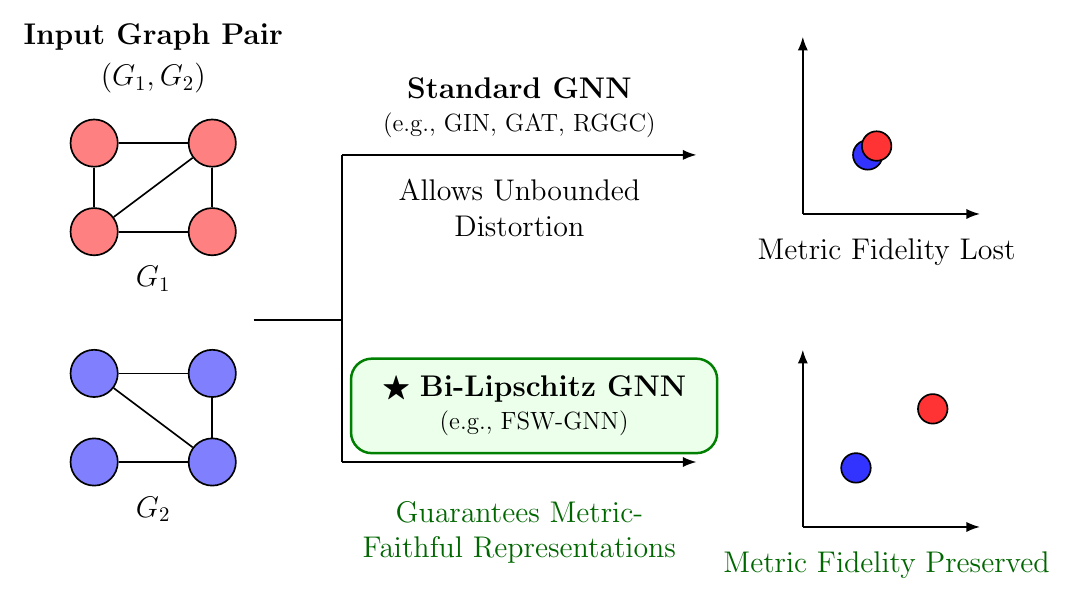}
    \vspace{-10pt}
    \caption{
    Motivation for geometrically rigorous GNN encoders in graph matching pipelines: Standard GNNs can distort structural distances, while bi-Lipschitz GNNs preserve metric fidelity in Euclidean space.
    }
    \label{fig:lede_motivation}
    \vspace{-10pt}
\end{wrapfigure}

We characterize this issue through the lens of metric distortion (Fig.~\ref{fig:lede_motivation}). At the graph level, structurally distinct graphs may be mapped too close together by similarity predictors~\citep{bai2020simgnn,ranjan2022greed}, weakening the fidelity of embedding-based GED surrogates. At the node level, matching-based estimators can inherit distorted cross-graph comparison costs before alignment is even performed~\citep{jain2024graph,tang2025fused}. Since many neural GED methods operate on graph-pair distances, node-matching scores, or relaxed alignment objectives derived from encoder outputs, the quality of these representations directly affects the faithfulness of the downstream estimator. We bridge ideas from invariant geometric machine learning, specifically bi-Lipschitz invariant theory~\citep{cahill2026approximation}, to neural graph matching in the \(1\)-WL regime considered in current benchmarks~\citep{liu2024graph,bai2020simgnn,pellizzoni2025gelato,xu2019powerful,xu2025graph,roy2025position}. We focus on two broad classes of methods: \textit{graph-level similarity predictors} and \textit{node-level matching-based estimators}. On fixed graph collections where the doubly-stochastic metric \(d_{\mathrm{DS}}\) is comparable to GED, we show that graph-level bi-Lipschitz encoder geometry yields tighter and more stable GED surrogates, while for matching-based methods, node-level bi-Lipschitz geometry propagates to encoder-induced alignment costs and the optimized alignment objective. Our goal is to demonstrate that geometry-aware encoders provide a strong and analyzable foundation for neural GED pipelines.

To instantiate this perspective, we adapt FSW-GNN~\citep{sverdlov2025fsw}, a WL-equivalent encoder with bi-Lipschitz guarantees with respect to \(d_{\mathrm{DS}}\), as a drop-in replacement in representative neural GED pipelines. Empirically, we observe that the resulting geometry-aware variants consistently improve GED prediction and ranking performance across multiple model families. Additional faithfulness analyses, transfer results, and ablations support the interpretation suggested by the theory: better encoder geometry improves the conditioning of the surrogate quantities from which neural GED is estimated.

\paragraph{Our Contributions.}
\begin{enumerate}
    \item We develop a geometric framework for neural GED estimation that unifies both graph-similarity predictors and matching-based estimators. For fixed \(1\)-WL-separable graph collections where \(d_{\mathrm{DS}}\) is comparable to GED, we show that encoder distortion directly governs the quality of the resulting GED surrogate. \textit{To the best of our knowledge, this is the first work to systematically relate encoder geometry to GED estimation performance.}
    
    \item For matching-based neural GED architectures, we prove that node-level bi-Lipschitz geometry propagates from encoder-induced node costs to relaxed alignment objectives, giving a precise mechanism by which encoder quality affects downstream edit estimation.
    
    \item We instantiate the theory with FSW-GNN~\citep{sverdlov2025fsw}, a bi-Lipschitz WL-equivalent encoder, as a drop-in replacement in representative neural GED architectures, without changing their downstream prediction or matching mechanisms. The resulting geometry-aware variants consistently and significantly improve GED prediction and ranking quality; faithfulness, transfer, and ablation results support the view that these gains arise from stronger representation geometry.
\end{enumerate}

%%%%%%%%%%%%%%%%%%%%%%%%%%%%%%%%%%%%
%% RW
%%%%%%%%%%%%%%%%%%%%%%%%%%%%%%%%%%%%

\section{Related Work}
\label{sec:rw}

\paragraph{Neural GED and graph matching.}
Neural GED methods are commonly divided into graph-similarity predictors and matching-based estimators~\citep{liu2024graph}. The first family includes models such as GMN, SimGNN, GraphSim, EGSC, ERIC, H2MN, and GraSP, which predict GED or graph similarity from learned graph-level representations~\citep{li2019graph,bai2020simgnn,bai2020learning,egsc,eric,h2mn,zheng2025grasp}. The second family includes models such as GEDGNN, GraphEDX, ISONET, and fused Gromov--Wasserstein approaches, which estimate similarity through explicit or relaxed alignments~\citep{piao2023computing,jain2024graph,isonet,tang2025fused}. Across both families, most work focuses on architectural design, prediction heads, or alignment objectives, while the geometry of the encoder itself is rarely analyzed.

\paragraph{Metric-aware neural graph matching.}
Several works have introduced stronger output-side or objective-side structure. GREED enforces metric properties at the prediction layer~\citep{ranjan2022greed}; GraphEDX and GEDAN learn more flexible edit-cost surrogates~\citep{jain2024graph,leonardi2026gedan}; and TaGSim decomposes similarity into operation-aware components~\citep{tagsim}. More recently, GRAIL explores an alternative program-synthesis route using LLMs~\citep{verma2025grail}. This perspective also connects to the broader graph matching literature on continuous relaxations, where both the risks of loose convex surrogates~\citep{lyzinski2014graph, dymds} and the value of tighter lifting-free relaxations~\citep{bernard2018ds} have been emphasized. Our perspective is complementary: rather than changing the head or matching objective, we study how the geometry of the upstream encoder affects the quality of the quantities these downstream modules consume.

\paragraph{Geometric representation learning.}
Recent work in invariant and geometric learning has moved beyond expressive power alone and started to analyze quantitative representation quality. Davidson and Dym show that standard sum-based multiset and graph networks satisfy only weaker lower-H\"older-type guarantees in expectation~\citep{davidsonHolderStabilityMultiset2024}, while FSW-GNN provides bi-Lipschitz guarantees with respect to WL-equivalent graph metrics, including the doubly-stochastic metric used here~\citep{sverdlov2025fsw}. Related work studies metric control for graph stability, permutation-invariant embeddings, and low-distortion orbit-space embeddings~\citep{chuang,dym2026quantitative,cahill2026approximation}. Our contribution brings these geometric ideas into neural GED estimation and uses them to analyze both graph-level prediction pipelines and matching-based architectures.

%%%%%%%%%%%%%%%%%%%%%%%%%%%%%%%%%%%%
%% Theory
%%%%%%%%%%%%%%%%%%%%%%%%%%%%%%%%%%%%

\section{Geometry of Neural GED}
\label{sec:theory}

Let \(\Omega \subset \mathbb{R}^d\) be compact, and let
\(\mathcal{G}^{\mathrm{WL}}_{\le N}\) denote a class of attributed graphs with at most \(N\) nodes that are pairwise separable by the \(1\)-WL test. Throughout this section, we work with a finite collection \(\mathcal G \subseteq \mathcal{G}^{\mathrm{WL}}_{\le N}\) consisting of the graphs appearing in the usual datasets, taken modulo graph isomorphism. Let \(d_{\mathrm{GED}}\) denote the uniform-cost graph edit
distance on \(\mathcal{G}^{\mathrm{WL}}_{\le N}\), and let \(d_{\mathrm{DS}}\) denote the generalized doubly-stochastic graph metric used in the bi-Lipschitz analysis of FSW-GNN~\citep{sverdlov2025fsw}. We state the main results in this section and defer the detailed proofs to Appendix~\ref{app:proofs}.

\paragraph{DS Metric} The metric \(d_{\mathrm{DS}}\) originates from the doubly-stochastic relaxation of graph isomorphism. For equal-size featureless graphs, graph isomorphism can be written as the existence of a permutation matrix \(P\) satisfying \(AP=PA'\), where \(A\) and \(A'\) are the adjacency matrices of the two graphs. Replacing the permutation constraint by
the convex hull of permutation matrices, namely the set of doubly-stochastic matrices, yields a continuous relaxation and leads to the DS metric by minimizing the residual of this relaxed commutation relation.\ \citet{sverdlov2025fsw} extend this construction to attributed graphs with varying node counts by jointly penalizing structural mismatch and node-feature mismatch over rectangular doubly-stochastic
couplings. A key result of~\citep{sverdlov2025fsw} is that this generalized \(d_{\mathrm{DS}}\) is a \(1\)-WL-equivalent graph metric on \(\mathcal{G}^{\mathrm{WL}}_{\le N}\): it vanishes exactly on \(1\)-WL-equivalent graph pairs. Moreover, on compact feature domains, FSW-GNN admits graph-level bi-Lipschitz
embeddings with respect to \(d_{\mathrm{DS}}\)~\citep{sverdlov2025fsw}. This makes
\(d_{\mathrm{DS}}\) a natural bridge between the discrete edit geometry of GED and the continuous Euclidean representation spaces used by neural GED models.

We treat the two broad classes of neural GED estimators separately. The first class consists of \textit{graph similarity predictors} that map each graph to a graph-level embedding and estimate GED from the resulting pair of vectors. The second class consists of \textit{matching-based estimators} that form a cross-graph alignment object from node embeddings and then score that object. This distinction is important because graph-level and node-level geometry enter these architectures in different ways.

We note that the ambient regime for the theory is therefore the \(1\)-WL-separable one. We chose this as the natural setting for the present analysis for two reasons. First, the encoder families
compared here are built from WL-type message-passing architectures and are designed to operate in that expressive regime~\citep{xu2019powerful,sverdlov2025fsw}. Second,
current neural GED methods are overwhelmingly evaluated on datasets that effectively lie within this regime, while recent work moving beyond it does so by explicitly increasing
expressive power, e.g. via \(3\)-WL-type architectures~\citep{xu2025graph,liu2024graph,
zheng2025grasp,moscatelli20253wl}. Thus, the point of our analysis is not to claim that \(1\)-WL geometry is universally sufficient for GED, but rather to understand how the \emph{quality} of encoder geometry within the regime \textit{used in practice} affects neural GED estimation. We leave extending the analysis beyond the \(1\)-WL setting for future work.

On a fixed finite collection of graphs, GED is often accessed in practice through a continuous surrogate geometry rather than directly. In this work, we choose \(d_{\mathrm{DS}}\) as that surrogate, because it is \(1\)-WL-equivalent and because
FSW-GNN provides explicit bi-Lipschitz control with respect to it~\citep{sverdlov2025fsw}. This makes \(d_{\mathrm{DS}}\) a natural bridge between the discrete edit metric and the continuous representation spaces used by neural GED models.

\begin{proposition}[GED--DS comparability]
\label{prop:ged_ds_dataset}
Define
\eq{
a_{\mathcal G}
:=
\min_{\substack{G,H\in\mathcal G\\G\neq H}}
\frac{d_{\mathrm{GED}}(G,H)}{d_{\mathrm{DS}}(G,H)},
\qquad
b_{\mathcal G}
:=
\max_{\substack{G,H\in\mathcal G\\G\neq H}}
\frac{d_{\mathrm{GED}}(G,H)}{d_{\mathrm{DS}}(G,H)}.
}
Then
\eq{
0<a_{\mathcal G}\le b_{\mathcal G}<\infty,
}
and, \(\forall G,H\in\mathcal G\),
\eq{
a_{\mathcal G}\, d_{\mathrm{DS}}(G,H)
\;\le\;
d_{\mathrm{GED}}(G,H)
\;\le\;
b_{\mathcal G}\, d_{\mathrm{DS}}(G,H).
}
\end{proposition}

Proposition~\ref{prop:ged_ds_dataset} says that on the fixed finite graph collection under study, \(d_{\mathrm{DS}}\) and \(d_{\mathrm{GED}}\) are comparable up to dataset-dependent constants \(a_{\mathcal G}\) and \(b_{\mathcal G}\). These constants are not controlled by the encoder or the downstream method; they only provide the bridge from the surrogate geometry \(d_{\mathrm{DS}}\) to \(d_{\mathrm{GED}}\). The encoder-side contribution enters separately through \(c_G\) and \(C_G\), so on a fixed dataset any improvement in surrogate conditioning must come from improved encoder geometry. This is precisely why FSW-GNN is relevant here, as it provides explicit bi-Lipschitz control with respect to \(d_{\mathrm{DS}}\), whereas standard sum-based MPNNs are only known to satisfy much weaker lower-Hölder-type guarantees in expectation~\citep{davidsonHolderStabilityMultiset2024}. We now state the encoder-side assumption that will be propagated through graph-level neural GED architectures.

\begin{assumption}[Graph-level bi-Lipschitz encoder]
\label{ass:graph_bilip}
The graph encoder \(\Phi : \mathcal G \to \mathbb{R}^m\) satisfies
\eq{
c_G\, d_{\mathrm{DS}}(G,H)
\;\le\;
\|\Phi(G)-\Phi(H)\|_2
\;\le\;
C_G\, d_{\mathrm{DS}}(G,H)
\qquad
\forall\, G,H \in \mathcal G,
}
for some constants \(0 < c_G \le C_G < \infty\).
\end{assumption}

Assumption~\ref{ass:graph_bilip} requires graph embeddings to preserve the \(d_{\mathrm{DS}}\) geometry up to controlled multiplicative distortion. We argue that this bi-Lipschitz property is precisely the geometric condition that makes encoder
representations useful for GED estimation: the lower bound prevents distant graphs from collapsing in the embedding space, while the upper bound ensures that small perturbations are not artificially amplified. For FSW-GNN-style encoders, this assumption is directly motivated by the graph-level bi-Lipschitz theorem of\ \citet{sverdlov2025fsw}. More broadly,
it serves as the encoder-side condition under which graph-pair prediction quality can be analyzed quantitatively. This is also the natural point of contrast with standard sum-based MPNNs, for which available results are substantially weaker, e.g. lower-Hölder-type guarantees in expectation whose quality degrades with depth~\citep{davidsonHolderStabilityMultiset2024}.

Taken together, Proposition~\ref{prop:ged_ds_dataset} and
Assumption~\ref{ass:graph_bilip} formalize a common intuition underlying most neural GED solvers through the DS–GED connection: Euclidean distances between graph embeddings encode quantitative information about GED. Crucially, this conclusion is not a consequence of expressive power alone; it depends on the metric conditioning of the encoder.

\subsection{Graph Similarity Predictors}

Graph similarity predictors estimate GED from graph-level representations, typically by a norm-based distance head or by an interaction head applied to a pair of embeddings. For this class of methods, the central question is \textit{how much of the target edit geometry survives the passage from graphs to graph embeddings}. Under Assumption~\ref{ass:graph_bilip} and with Proposition~\ref{prop:ged_ds_dataset}, we obtain explicit answers.

\begin{proposition}[Graph-level bounds]
\label{prop:graph_sandwich}
For every \(G,H \in \mathcal G\),
\eq{
\frac{a_{\mathcal G}}{C_G}\,\|\Phi(G)-\Phi(H)\|_2
\;\le\;
d_{\mathrm{GED}}(G,H)
\;\le\;
\frac{b_{\mathcal G}}{c_G}\,\|\Phi(G)-\Phi(H)\|_2.
}
\end{proposition}

Proposition~\ref{prop:graph_sandwich} shows that graph-embedding distances provide a multiplicatively controlled surrogate for GED on the considered graph datasets. The quality of this surrogate is governed by two distinct sources of distortion: the
comparability of \(d_{\mathrm{DS}}\) to GED through \(a_{\mathcal G}\) and \(b_{\mathcal G}\), and the graph-level distortion of the encoder through \(c_G\) and \(C_G\). In particular, improved graph geometry tightens the surrogate independently of the specific prediction head used downstream.

\begin{theorem}[Best scalar calibration for norm-based predictors]
\label{thm:calibration}
Let
\eq{
\alpha := \frac{a_{\mathcal G}}{C_G},
\quad
\beta := \frac{b_{\mathcal G}}{c_G},
\quad
\gamma^\star := \frac{\alpha+\beta}{2}.
}
Then, for every \(G,H \in \mathcal G\),
\eq{
\left|
\gamma^\star \|\Phi(G)-\Phi(H)\|_2 - d_{\mathrm{GED}}(G,H)
\right|
\le
\frac{\beta-\alpha}{2}\,\|\Phi(G)-\Phi(H)\|_2.
}
\end{theorem}

Through Theorem~\ref{thm:calibration}, we formalize the simplest graph-level prediction mechanism: a scalar multiple of embedding distance, and show that even such a minimal head admits a uniform error band, whose width is controlled by the same geometric constants as in Proposition~\ref{prop:graph_sandwich}. Thus, improved encoder geometry directly narrows the uncertainty band for norm-based GED prediction, rather than being merely a qualitative desideratum.

\paragraph{Ranking interpretation for graph-level predictors.}
Theorem~\ref{thm:calibration} controls absolute approximation error for norm-based GED surrogates. For rank-based evaluation, however, the relevant question is whether the ordering of graph pairs induced by the predictor agrees with the ordering induced by the true GED. In the next corollary, we identify a regime in which this agreement is guaranteed.

\begin{corollary}[Ranking preservation for monotone norm-head predictors]
\label{cor:ranking_graph}
Let \(\widehat d_\psi(G,H):=\psi\!\left(\|\Phi(G)-\Phi(H)\|_2\right)\), where \(\psi:\mathbb{R}_{\ge 0}\to\mathbb{R}\) is strictly increasing. Define
\eq{
\kappa_G := \frac{\beta}{\alpha}
= \frac{b_{\mathcal G} C_G}{a_{\mathcal G} c_G}.
}
Then for any two graph pairs \((G_1,H_1)\) and \((G_2,H_2)\) (and similarly with $<$ and $\kappa_G^{-1}$),
\eq{
d_{\mathrm{GED}}(G_1,H_1)
>
\kappa_G\, d_{\mathrm{GED}}(G_2,H_2)
\quad\Longrightarrow\quad
\widehat d_\psi(G_1,H_1)>\widehat d_\psi(G_2,H_2).
}
\end{corollary}

Corollary~\ref{cor:ranking_graph} provides a pairwise order-preservation guarantee for graph-level GED predictors. It shows that any pairwise comparison, whose true GED values differ by a factor larger than \(\kappa_G\), must be ordered correctly by any strictly increasing norm-head predictor. The smaller the distortion ratio \(\kappa_G\), the larger the set of pairwise comparisons for which correct ranking is certified. This directly supports rank-based evaluation metrics: it increases the set of pairwise comparisons that are necessarily concordant, thereby improving the conditioning of Kendall's \(\tau\) and, more broadly, Spearman's \(\rho\) and retrieval metrics such as \(P@10\) and \(P@20\). By contrast, absolute calibration metrics such as MAE are governed by Theorem~\ref{thm:calibration}, which controls the numerical error of the predicted score rather than its induced ordering. Accordingly, improved graph-level geometry should be expected to benefit ranking metrics even when residual
calibration error remains. This gives a concrete graph-level mechanism for improvements in retrieval and rank-correlation measures.

Many graph similarity predictors use an interaction module, MLP, or neural tensor style head rather than a single Euclidean norm~\citep{bai2020simgnn,li2019graph,ranjan2022greed,zheng2025grasp}. The next proposition shows that once the encoder is bi-Lipschitz with respect to the surrogate geometry, any downstream head that is Lipschitz on embedding pairs inherits a corresponding stability guarantee.

\begin{proposition}[Stability of Lipschitz interaction heads]
\label{prop:lipschitz_head}
Let \(h:\mathbb{R}^m\times\mathbb{R}^m\to\mathbb{R}\) be \(L_h\)-Lipschitz with respect to \(\|(u,v)\|_{\oplus} := \|u\|_2+\|v\|_2\). Define \(s_h(G,H) := h(\Phi(G),\Phi(H))\). Then, for every \(G,H,G',H' \in \mathcal G\),
\eq{
|s_h(G,H)-s_h(G',H')|
\le
L_h C_G\,
\bigl(
d_{\mathrm{DS}}(G,G')+d_{\mathrm{DS}}(H,H')
\bigr).
}
\end{proposition}

Proposition~\ref{prop:lipschitz_head} extends the graph-level analysis beyond pure distance heads. It provides a class-level stability statement for graph similarity predictors such as SimGNN~\citep{bai2020simgnn}, GMN~\citep{li2019graph}, GREED~\citep{ranjan2022greed}, and GraSP-style~\citep{zheng2025grasp} models: once the encoder preserves the surrogate geometry with controlled distortion, any Lipschitz interaction head inherits a corresponding control with respect to perturbations measured in \(d_{\mathrm{DS}}\).

\subsection{Matching-based Estimators}

We now turn to neural GED estimators that build a cross-graph alignment object from node embeddings and then score that object. This includes OT-based, assignment-based, and edit-cost-based architectures in which the encoder influences GED estimation through the node-level costs supplied to the matching module, e.g., GEDGNN~\citep{piao2023computing}, GraphEDX~\citep{jain2024graph}, FGWAlign~\citep{tang2025fused}. In this setting, the relevant geometric question is not only whether whole-graph embeddings are well conditioned, but also whether cross-graph node discrepancies are faithfully represented.

Following standard assignment- and OT-based GED formulations, we fix once and for all an equalization rule that maps each pair \((G,H)\in\mathcal G\times\mathcal G\) to an equalized pair \((\bar G,\bar H)\), obtained for instance by dummy-node or virtual-node
augmentation, such that \(|V(\bar G)| = |V(\bar H)|\). This is also the setting in which FGW-based GED formulations can be stated most cleanly. For a fixed pair \((G,H)\), let us write \(n_{\bar G,\bar H} := |V(\bar G)| = |V(\bar H)|\), and define the set of doubly-stochastic alignment plans
\eq{
\mathcal{D}_{\bar G,\bar H}
:=
\left\{
P\in\mathbb{R}_{\ge 0}^{n_{\bar G,\bar H}\times n_{\bar G,\bar H}}
:\;
P\mathbf{1}=\mathbf{1},
\;
P^\top\mathbf{1}=\mathbf{1}
\right\}.
}
We work with unnormalized doubly-stochastic plans. Relative to the OT convention with uniform marginals \(1/n_{\bar G,\bar H}\), this amounts only to a constant rescaling and does not affect the distortion arguments below.

Let \(\delta_{\bar G,\bar H}:V(\bar G)\times V(\bar H)\to\mathbb{R}_{\ge 0}\) be a reference node-discrepancy function, and let \(\phi\) be a node encoder producing
embeddings, \(\phi_{\bar G}(i),\phi_{\bar H}(j)\in\mathbb{R}^p\).

\begin{assumption}[Node-level bi-Lipschitz encoder]
\label{ass:node_bilip}
There exist constants \(0 < c_V \le C_V < \infty\) such that for every pair
\((G,H)\in\mathcal G\times\mathcal G\), every equalized pair \((\bar G,\bar H)\)
induced by the fixed equalization rule, and all \(i\in V(\bar G)\), \(j\in V(\bar H)\),
\eq{
c_V\, \delta_{\bar G,\bar H}(i,j)
\;\le\;
\|\phi_{\bar G}(i)-\phi_{\bar H}(j)\|_2
\;\le\;
C_V\, \delta_{\bar G,\bar H}(i,j).
}
\end{assumption}

Assumption~\ref{ass:node_bilip} is the node-level analogue of
Assumption~\ref{ass:graph_bilip}. It requires the encoder to preserve a reference cross-graph node geometry up to controlled multiplicative distortion. This is the relevant hypothesis for matching-based estimators because these architectures do not
predict GED from a single graph-level distance alone; instead, they construct an alignment score from node-level comparison costs. The reference discrepancy \(\delta_{\bar G,\bar H}\) should be understood as a task-dependent surrogate cross-graph node cost induced by the graph family, the equalization scheme, and the matching architecture. As with Assumption~\ref{ass:graph_bilip}, this assumption is also consistent with recent geometric GNN theory, where FSW-GNN is shown to admit bi-Lipschitz guarantees with respect to WL-equivalent graph metrics and, under additional conditions, tree-based node metrics that quantify local message-passing geometry~\citep{sverdlov2025fsw, chuang}.

Now, let us define the node-cost matrices, \(C^\delta_{\bar G,\bar H}(i,j):=\delta_{\bar G,\bar H}(i,j)\), and \(C^\phi_{\bar G,\bar H}(i,j):=\|\phi_{\bar G}(i)-\phi_{\bar H}(j)\|_2\). For \(P\in\mathcal{D}_{\bar G,\bar H}\), let
\eq{
L_\delta^{\bar G,\bar H}(P)
:=
\langle P, C^\delta_{\bar G,\bar H}\rangle
=
\sum_{i\in V(\bar G)}\sum_{j\in V(\bar H)}
P_{ij}\,\delta_{\bar G,\bar H}(i,j),
}
\eq{
L_\phi^{\bar G,\bar H}(P)
:=
\langle P, C^\phi_{\bar G,\bar H}\rangle
=
\sum_{i\in V(\bar G)}\sum_{j\in V(\bar H)}
P_{ij}\,\|\phi_{\bar G}(i)-\phi_{\bar H}(j)\|_2.
}

\begin{lemma}[Plan-wise distortion of alignment costs]
\label{lem:planwise}
Under Assumption~\ref{ass:node_bilip}, for every pair \((G,H)\in\mathcal G\times\mathcal G\)
with equalized versions \((\bar G,\bar H)\), and every
\(P\in\mathcal{D}_{\bar G,\bar H}\),
\eq{
c_V\, L_\delta^{\bar G,\bar H}(P)
\;\le\;
L_\phi^{\bar G,\bar H}(P)
\;\le\;
C_V\, L_\delta^{\bar G,\bar H}(P).
}
\end{lemma}

Lemma~\ref{lem:planwise} tells us that the distortion incurred by the encoder at the level of individual node pairs lifts directly to every admissible alignment plan. Thus, before any optimization is performed, the encoder already constrains the full family of relaxed matching costs seen by the downstream solver.

\begin{theorem}[Distortion bound for alignment objectives]
\label{thm:alignment}
For each pair \((G,H)\in\mathcal G\times\mathcal G\), let \(\mathcal{R}_{\bar G,\bar H}:\mathcal{D}_{\bar G,\bar H}\to\mathbb{R}\) be any structural term independent of the encoder, and let \(\lambda\ge 0\). Define
\eq{
J_\delta^{\bar G,\bar H}(P)
:=
\mathcal{R}_{\bar G,\bar H}(P)+\lambda L_\delta^{\bar G,\bar H}(P), \quad
J_\phi^{\bar G,\bar H}(P)
:=
\mathcal{R}_{\bar G,\bar H}(P)+\lambda L_\phi^{\bar G,\bar H}(P).
}
Then, for every pair \((G,H)\in\mathcal G\times\mathcal G\),
\[
\min_{P\in\mathcal{D}_{\bar G,\bar H}}
\Bigl(
\mathcal{R}_{\bar G,\bar H}(P)+\lambda c_V L_\delta^{\bar G,\bar H}(P)
\Bigr)
\;\le\;
\min_{P\in\mathcal{D}_{\bar G,\bar H}} J_\phi^{\bar G,\bar H}(P)
\;\le\;
\min_{P\in\mathcal{D}_{\bar G,\bar H}}
\Bigl(
\mathcal{R}_{\bar G,\bar H}(P)+\lambda C_V L_\delta^{\bar G,\bar H}(P)
\Bigr).
\]
\end{theorem}

For alignment-based neural GED estimators, Theorem~\ref{thm:alignment} shows that encoder distortion propagates all the way to the \emph{optimized} alignment score, not merely to individual pairwise costs. We have kept the theorem general in the structural term \(\mathcal{R}_{\bar G,\bar H}\), so as to cover architectures in which node costs are combined with adjacency consistency, transport regularization, or other structural penalties. In brief, it implies that improved node geometry improves the faithfulness of the optimization problem solved by the matching module.

\begin{corollary}[Comparability transfer for alignment surrogates]
\label{cor:alignment_transfer}
Assume, in addition, that, for the fixed equalization rule and reference discrepancy above, there exist constants \(0<\eta_{1,\mathcal G}\le \eta_{2,\mathcal G}<\infty\), such that, for every pair \((G,H)\in\mathcal G\times\mathcal G\),
\eq{
\eta_{1,\mathcal G}\, d_{\mathrm{GED}}(G,H)
\;\le\;
\min_{P\in\mathcal{D}_{\bar G,\bar H}} J_\delta^{\bar G,\bar H}(P)
\;\le\;
\eta_{2,\mathcal G}\, d_{\mathrm{GED}}(G,H).
}
Moreover, assume that \(\mathcal{R}_{\bar G,\bar H}(P)\ge 0\) for all
\((G,H)\in\mathcal G\times\mathcal G\) and all \(P\in\mathcal{D}_{\bar G,\bar H}\).
Let \(\underline c := \min\{1,c_V\}, \quad \overline C := \max\{1,C_V\}\). Then, for every pair \((G,H)\in\mathcal G\times\mathcal G\),
\eq{
\underline c\,\eta_{1,\mathcal G}\, d_{\mathrm{GED}}(G,H)
\;\le\;
\min_{P\in\mathcal{D}_{\bar G,\bar H}} J_\phi^{\bar G,\bar H}(P)
\;\le\;
\overline C\,\eta_{2,\mathcal G}\, d_{\mathrm{GED}}(G,H).
}
\end{corollary}

Corollary~\ref{cor:alignment_transfer} makes explicit how node-level geometry connects back to GED once the reference alignment objective is itself a valid surrogate for edit distance on the graph family under study. The statement isolates two
ingredients: the quality of the alignment surrogate \(J_\delta^{\bar G,\bar H}\), through \(\eta_{1,\mathcal G}\) and
\(\eta_{2,\mathcal G}\), and the quality of the encoder geometry, through \(c_V\) and \(C_V\), enabling us to cleanly analyze neural GED models that infer edit distance through relaxed matching objectives.

The graph-level and alignment-level results are complementary. For graph similarity predictors, encoder geometry determines how graph-pair embeddings reflect a continuous surrogate for GED. For matching-based estimators, it determines node costs and the alignment objective used to estimate GED. In both cases, we have shown that metric conditioning can improve performance, though the intervention differs across architectures. In Appendix~\ref{app:connect}, we show how the graph-level results apply to predictors such as GREED~\citep{ranjan2022greed} and GraSP~\citep{zheng2025grasp}, while the alignment-level results apply to estimators such as GEDGNN~\citep{piao2023computing} and GraphEDX~\citep{jain2024graph}. Together, these results provide a unified geometric framework explaining why better-conditioned encoder geometry can improve neural GED estimation across architectures without changing downstream prediction or matching mechanisms, as shown empirically in the next section.

%%%%%%%%%%%%%%%%%%%%%%%%%%%%%%%%%%%%
%% Experiments & Discussion
%%%%%%%%%%%%%%%%%%%%%%%%%%%%%%%%%%%%

\section{Experiments}
\label{sec:exp}

We evaluate whether improving encoder geometry can systematically improve neural GED estimation across both graph-level similarity predictors and matching/alignment-based estimators. To do so, we replace the original message-passing encoder and aggregation layers in each baseline with FSW-GNN~\citep{sverdlov2025fsw} layers while preserving the downstream prediction or matching module as closely as possible. This design isolates encoder geometry from the rest of the pipeline and aligns the empirical study with the distinction in Section~\ref{sec:theory} between graph-level surrogates and alignment-based estimators.

\paragraph{Datasets.}
All experiments follow the leak-free evaluation protocol introduced in Gelato~\citep{pellizzoni2025gelato} and the benchmark guidance of \citet{roy2025position} to prevent train--test leakage through isomorphic or structurally duplicated graphs. In the main text, we report primary results on four representative datasets: AIDS, IMDB-16, \textsc{molhiv}-16, and \textsc{code2}-22, covering molecular, social-network, and code graphs. Results on LINUX and ZINC-16, together with fuller metric breakdowns and additional implementation details, are deferred to Appendix~\ref{app:add-exp}. Unless otherwise stated, we use node features when available and omit edge features, since several of the compared baselines do not support them consistently.

\paragraph{Metrics.}
In the results here, we emphasize mean absolute error (MAE) and Kendall's rank correlation $\tau$. MAE measures numerical fidelity to the target GED, while $\tau$ captures the quality of pairwise ordering, which is especially relevant for retrieval and nearest-neighbor use cases. Appendix~\ref{app:add-exp} additionally reports Spearman's $\rho$, P@10, and P@20, which confirm the same trends more broadly.

\paragraph{Baselines and BL variants.}
Our baselines span the two architectural families analyzed in Section~\ref{sec:theory}. For graph-similarity predictors, we consider SimGNN~\citep{bai2020simgnn}, GMN~\citep{li2019graph}, EGSC~\citep{egsc}, ERIC~\citep{eric}, GREED~\citep{ranjan2022greed}, and GraSP~\citep{zheng2025grasp}. For matching/alignment-based estimators, we consider GEDGNN~\citep{piao2023computing} and GraphEdX~\citep{jain2024graph}. For each method, the suffix ``BL'' denotes the geometry-aware variant obtained by replacing the original encoder and aggregation layers with FSW-GNN layers while leaving the downstream head and training objective unchanged whenever possible.

\paragraph{Training details.}
We use author-recommended hyperparameters whenever possible and intentionally avoid aggressive retuning after encoder replacement, so that the observed gains can be attributed primarily to the change in encoder geometry rather than to a larger model-specific search budget. Original and BL variants are matched as closely as possible in hidden dimension, optimizer, training budget, early stopping, and data splits; unless a method requires a documented exception, BL models use hidden dimension $64$. All results are reported as mean $\pm$ standard deviation over $10$ runs. More implementation details are given in Appendix~\ref{app:add-exp}.

\subsection{Main results}

Table~\ref{tab:main-results} reports the primary comparison between each baseline and its BL counterpart on AIDS, IMDB-16, \textsc{molhiv}-16, and \textsc{code2}-22. The central empirical question is whether improving encoder geometry yields consistent gains across model families and evaluation criteria. The answer is affirmative: across all eight baselines and all four datasets, replacing the encoder with its BL counterpart improves both MAE and $\tau$. This trend holds for both architectural families studied in our theory. Among graph-level predictors, GraSP-BL attains the best MAE on all four main-text datasets, improving over GraSP. At the same time, EGSC-BL and ERIC-BL remain especially strong on ranking metrics, indicating that once the encoder geometry is improved, different downstream heads can still induce different trade-offs between regression fidelity and ranking fidelity. For the matching-based methods, the picture is similar but more moderate: GraphEdX-BL and GEDGNN-BL improve over their non-BL counterparts on every reported dataset and metric, showing that the encoder replacement helps not only pooled graph predictors but also models that estimate GED through learned alignment. Our results and ablations in Appendix~\ref{app:add-exp} further affirm this conclusion.

\begin{table}[t]
\caption{
GED prediction performance across datasets. Due to space limits, results on LINUX, ZINC-16, and additional metrics ($\rho$, P@10, P@20) are deferred to Appendix~\ref{app:add-exp}. Lower is better for MAE and higher is better for $\tau$.
}
\label{tab:main-results}
\centering
\small

\begin{adjustbox}{width=\textwidth}
\begin{tabular}{lcccccccc}
\toprule

\multirow{2.4}{*}{Method}
& \multicolumn{2}{c}{AIDS}
& \multicolumn{2}{c}{IMDB-16}
& \multicolumn{2}{c}{\textsc{molhiv}-16}
& \multicolumn{2}{c}{\textsc{code2}-22} \\

\cmidrule(lr){2-3}
\cmidrule(lr){4-5}
\cmidrule(lr){6-7}
\cmidrule(lr){8-9}

& MAE$\downarrow$ & $\tau\uparrow$
& MAE$\downarrow$ & $\tau\uparrow$
& MAE$\downarrow$ & $\tau\uparrow$
& MAE$\downarrow$ & $\tau\uparrow$ \\

\midrule

SimGNN
& $0.691\scriptstyle{\pm0.018}$ & $0.721\scriptstyle{\pm0.011}$
& $0.912\scriptstyle{\pm0.051}$ & $0.748\scriptstyle{\pm0.014}$
& $0.792\scriptstyle{\pm0.031}$ & $0.662\scriptstyle{\pm0.011}$
& $0.781\scriptstyle{\pm0.022}$ & $0.694\scriptstyle{\pm0.013}$ \\

SimGNN-BL
& $0.298\scriptstyle{\pm0.011}$ & $0.748\scriptstyle{\pm0.010}$
& $0.296\scriptstyle{\pm0.013}$ & $0.781\scriptstyle{\pm0.012}$
& $0.362\scriptstyle{\pm0.018}$ & $0.689\scriptstyle{\pm0.010}$
& $0.221\scriptstyle{\pm0.011}$ & $0.731\scriptstyle{\pm0.011}$ \\

GMN
& $0.676\scriptstyle{\pm0.019}$ & $0.729\scriptstyle{\pm0.011}$
& $0.894\scriptstyle{\pm0.054}$ & $0.756\scriptstyle{\pm0.013}$
& $0.775\scriptstyle{\pm0.029}$ & $0.671\scriptstyle{\pm0.010}$
& $0.763\scriptstyle{\pm0.021}$ & $0.702\scriptstyle{\pm0.012}$ \\

GMN-BL
& $0.317\scriptstyle{\pm0.012}$ & $0.754\scriptstyle{\pm0.009}$
& $0.388\scriptstyle{\pm0.012}$ & $0.786\scriptstyle{\pm0.011}$
& $0.391\scriptstyle{\pm0.017}$ & $0.697\scriptstyle{\pm0.009}$
& $0.238\scriptstyle{\pm0.010}$ & $0.738\scriptstyle{\pm0.011}$ \\

EGSC
& $0.284\scriptstyle{\pm0.010}$ & $0.792\scriptstyle{\pm0.008}$
& $0.451\scriptstyle{\pm0.029}$ & $0.829\scriptstyle{\pm0.009}$
& $0.271\scriptstyle{\pm0.016}$ & $0.756\scriptstyle{\pm0.008}$
& $0.189\scriptstyle{\pm0.009}$ & $0.801\scriptstyle{\pm0.009}$ \\

EGSC-BL
& $0.132\scriptstyle{\pm0.006}$ & $0.809\scriptstyle{\pm0.007}$
& $0.058\scriptstyle{\pm0.007}$ & $0.848\scriptstyle{\pm0.008}$
& $0.139\scriptstyle{\pm0.008}$ & $0.772\scriptstyle{\pm0.007}$
& $0.071\scriptstyle{\pm0.004}$ & $0.818\scriptstyle{\pm0.008}$ \\

ERIC
& $0.312\scriptstyle{\pm0.010}$ & $0.781\scriptstyle{\pm0.008}$
& $0.492\scriptstyle{\pm0.031}$ & $0.818\scriptstyle{\pm0.010}$
& $0.294\scriptstyle{\pm0.017}$ & $0.742\scriptstyle{\pm0.009}$
& $0.201\scriptstyle{\pm0.010}$ & $0.789\scriptstyle{\pm0.010}$ \\

ERIC-BL
& $0.148\scriptstyle{\pm0.006}$ & $0.801\scriptstyle{\pm0.007}$
& $0.063\scriptstyle{\pm0.008}$ & $0.839\scriptstyle{\pm0.008}$
& $0.151\scriptstyle{\pm0.008}$ & $0.763\scriptstyle{\pm0.008}$
& $0.079\scriptstyle{\pm0.005}$ & $0.811\scriptstyle{\pm0.009}$ \\

GREED
& $0.794\scriptstyle{\pm0.021}$ & $0.677\scriptstyle{\pm0.014}$
& $1.054\scriptstyle{\pm0.066}$ & $0.721\scriptstyle{\pm0.017}$
& $0.861\scriptstyle{\pm0.036}$ & $0.618\scriptstyle{\pm0.014}$
& $0.998\scriptstyle{\pm0.028}$ & $0.651\scriptstyle{\pm0.016}$ \\

GREED-BL
& $0.612\scriptstyle{\pm0.018}$ & $0.702\scriptstyle{\pm0.012}$
& $0.684\scriptstyle{\pm0.047}$ & $0.749\scriptstyle{\pm0.015}$
& $0.694\scriptstyle{\pm0.029}$ & $0.646\scriptstyle{\pm0.012}$
& $0.721\scriptstyle{\pm0.022}$ & $0.689\scriptstyle{\pm0.014}$ \\

GraSP
& $\mathbf{0.061\scriptstyle{\pm0.004}}$ & $0.654\scriptstyle{\pm0.012}$
& $1.205\scriptstyle{\pm0.071}$ & $0.538\scriptstyle{\pm0.018}$
& $\mathbf{0.139\scriptstyle{\pm0.008}}$ & $0.724\scriptstyle{\pm0.016}$
& $\mathbf{0.104\scriptstyle{\pm0.006}}$ & $0.571\scriptstyle{\pm0.017}$ \\

GraSP-BL
& $\mathbf{0.028\scriptstyle{\pm0.002}}$ & $0.718\scriptstyle{\pm0.010}$
& $0.114\scriptstyle{\pm0.009}$ & $0.792\scriptstyle{\pm0.010}$
& $\mathbf{0.069\scriptstyle{\pm0.004}}$ & $0.754\scriptstyle{\pm0.014}$
& $\mathbf{0.032\scriptstyle{\pm0.002}}$ & $0.681\scriptstyle{\pm0.012}$ \\

\midrule

GraphEDX
& $0.572\scriptstyle{\pm0.013}$ & $0.744\scriptstyle{\pm0.010}$
& $3.180\scriptstyle{\pm0.185}$ & $0.771\scriptstyle{\pm0.012}$
& $0.638\scriptstyle{\pm0.023}$ & $0.698\scriptstyle{\pm0.010}$
& $0.429\scriptstyle{\pm0.015}$ & $0.732\scriptstyle{\pm0.011}$ \\

GraphEDX-BL
& $0.537\scriptstyle{\pm0.012}$ & $0.768\scriptstyle{\pm0.009}$
& $2.842\scriptstyle{\pm0.162}$ & $0.803\scriptstyle{\pm0.010}$
& $0.600\scriptstyle{\pm0.021}$ & $0.726\scriptstyle{\pm0.009}$
& $0.394\scriptstyle{\pm0.013}$ & $0.761\scriptstyle{\pm0.010}$ \\

GEDGNN
& $1.782\scriptstyle{\pm0.044}$ & $0.432\scriptstyle{\pm0.021}$
& $0.348\scriptstyle{\pm0.069}$ & $0.281\scriptstyle{\pm0.029}$
& $4.914\scriptstyle{\pm0.030}$ & $0.298\scriptstyle{\pm0.022}$
& $2.273\scriptstyle{\pm0.080}$ & $0.317\scriptstyle{\pm0.025}$ \\

GEDGNN-BL
& $1.421\scriptstyle{\pm0.037}$ & $0.471\scriptstyle{\pm0.018}$
& $0.221\scriptstyle{\pm0.042}$ & $0.352\scriptstyle{\pm0.025}$
& $3.982\scriptstyle{\pm0.028}$ & $0.334\scriptstyle{\pm0.020}$
& $1.624\scriptstyle{\pm0.061}$ & $0.372\scriptstyle{\pm0.022}$ \\

\bottomrule
\end{tabular}
\end{adjustbox}
\end{table}

\subsection{Faithfulness Case Study: Assessing impact without model training}

Figure~\ref{fig:faithfulness} helps explain the downstream gains at the representation level. Even without any task-specific training, replacing the original GraSP encoder with its BL counterpart reduces the latent-distance MAE on every dataset: by $53.7\%$ on AIDS, $25.6\%$ on LINUX, $90.5\%$ on IMDB-16, $50.1\%$ on \textsc{molhiv}-16, $69.6\%$ on \textsc{code2}-22, and $33.4\%$ on ZINC-16. At the same time, the Spearman correlation between latent distances and graph similarity increases on all six datasets, by $5.3\%$, $5.0\%$, $23.9\%$, $4.8\%$, $9.4\%$, and $2.2\%$, respectively. Thus, before any end-to-end optimization, the BL encoder already produces a latent geometry that is both more quantitatively faithful to GED and more reliable for ranking graph pairs. This outcome matches the benchmark results in Table~\ref{tab:main-results}. GraSP-BL improves over GraSP on every dataset and on every metric reported in the main paper and appendix, with especially large gains on IMDB-16 and \textsc{code2}-22, precisely where the faithfulness gap is largest. The empirical picture is therefore coherent: when the embedding geometry better preserves structural dissimilarity, the induced GED surrogate becomes easier to calibrate and more reliable for retrieval and ranking.

\begin{wrapfigure}{r}{0.5\textwidth}
    \centering
    \vspace{-10pt}
    \includegraphics[width=0.48\textwidth]{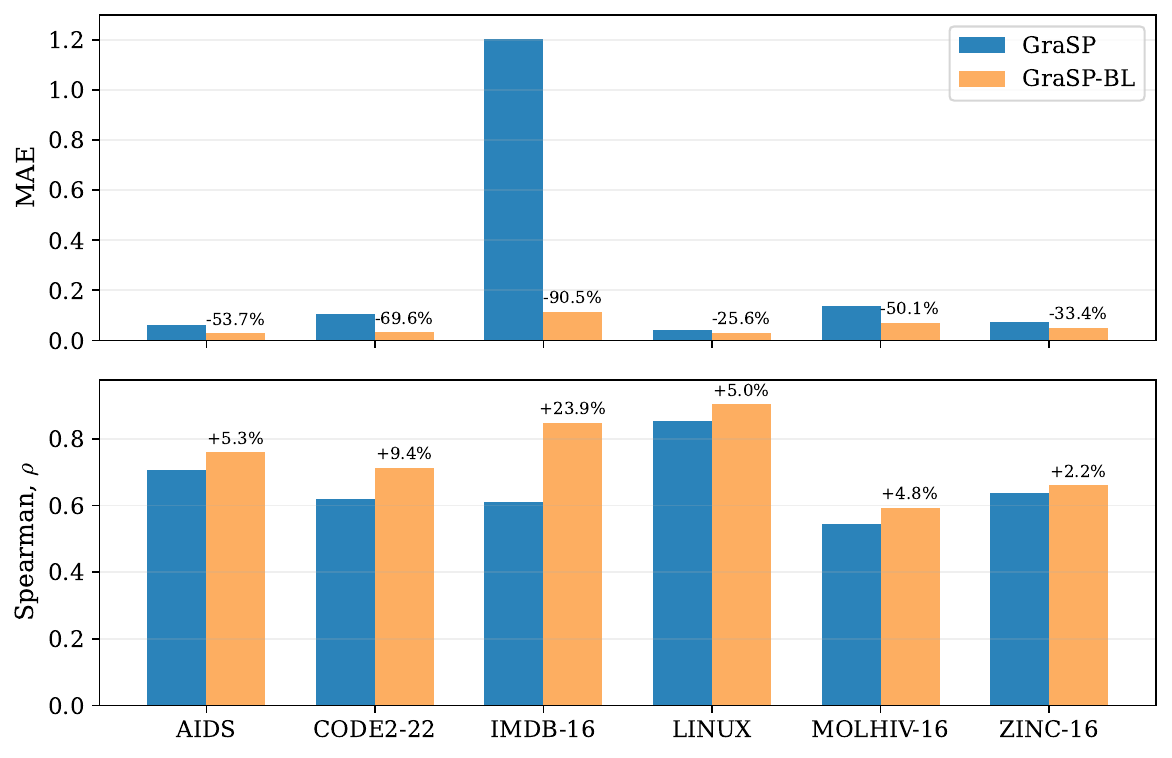}
    \vspace{-10pt}
    \caption{
    Faithfulness analysis of untrained GraSP variants. Across all datasets, the Bi-Lipschitz GNN better aligns embedding geometry with graph dissimilarity than the standard GNN (RGGC~\citep{rggc}).
    }
    \label{fig:faithfulness}
    \vspace{-10pt}
\end{wrapfigure}

\subsection{Connecting to theory}

Our empirical findings align closely with the theory in Section~\ref{sec:theory}. For graph-level predictors, the theory states that a bi-Lipschitz encoder yields a tighter and more stable surrogate for GED; Figure~\ref{fig:faithfulness} measures this effect directly by showing that latent distances track graph dissimilarity more faithfully even before training. For matching-based estimators, the theory predicts that improved node-level geometry should propagate to encoder-induced alignment costs and hence to the optimized matching objective. Empirically, both effects are visible: graph-level methods show especially strong improvements, while alignment-based methods also improve consistently once the encoder is replaced. Additional results and discussion in Appendix~\ref{app:add-exp} further validate this.

%%%%%%%%%%%%%%%%%%%%%%%%%%%%%%%%%%%%
%% Conclusion
%%%%%%%%%%%%%%%%%%%%%%%%%%%%%%%%%%%%

\section{Conclusion}
\label{sec:concl}

We developed a geometric framework for neural GED estimation showing that, on the finite 1-WL regime considered in this work, improved encoder geometry yields better-conditioned surrogates for both graph-similarity predictors and matching/alignment-based estimators. Instantiating this perspective with FSW-GNN~\citep{sverdlov2025fsw} as a drop-in bi-Lipschitz encoder replacement, we observed consistent gains across existing neural GED architectures in prediction, ranking, retrieval, and transfer, while the faithfulness analysis and ablations support the interpretation that these gains arise from stronger representation geometry rather than only downstream optimization. Overall, these results position geometry-aware encoders as a principled design choice for neural graph matching pipelines, while also making clear that they are not a complete solution: the present theory is restricted to the 1-WL setting and dataset-dependent surrogate comparability, and the current experiments do not cover edge labels, richer edit-cost settings, or the full efficiency--expressivity design space. We discuss these limitations and related future directions in Appendix~\ref{app:limitations}, and broader-impact considerations in Appendix~\ref{app:broader-impact}.

% \begin{ack}
% Use unnumbered first level headings for the acknowledgments. All acknowledgments
% go at the end of the paper before the list of references. Moreover, you are required to declare
% funding (financial activities supporting the submitted work) and competing interests (related financial activities outside the submitted work).
% More information about this disclosure can be found at: \url{https://neurips.cc/Conferences/2026/PaperInformation/FundingDisclosure}.

% Do {\bf not} include this section in the anonymized submission, only in the final paper. You can use the \texttt{ack} environment provided in the style file to automatically hide this section in the anonymized submission.
% \end{ack}

% \section*{References}
\bibliography{neurips_2026.bib}
\bibliographystyle{unsrtnat} % plainnat

%%%%%%%%%%%%%%%%%%%%%%%%%%%%%%%%%%%%%%%%%%%%%%%%%%%%%%%%%%%%

\appendix

% \section{Technical appendices and supplementary material}
% Technical appendices with additional results, figures, graphs, and proofs may be submitted with the paper submission before the full submission deadline (see above). You can upload a ZIP file for videos or code, but do not upload a separate PDF file for the appendix. There is no page limit for the technical appendices. 

% Note: Think of the appendix as ``optional reading'' for reviewers. The paper must be able to stand alone without the appendix; for example, adding critical experiments that support the main claims to an appendix is inappropriate. 

\section{Broader Impact}
\label{app:broader-impact}

This work is primarily foundational. It studies how encoder geometry affects the quality of neural graph edit distance estimation and related graph matching surrogates, rather than proposing a directly deployed application system. The most immediate positive impact is methodological: a clearer geometric understanding of when neural graph matching pipelines should be expected to produce more faithful distances, rankings, and alignments. In application areas where graph matching is already used, such as molecular retrieval, program analysis, and structured search, improved metric faithfulness may translate into more reliable candidate ranking and more interpretable matching behavior.

At the same time, stronger graph matching can have negative downstream uses. Better graph similarity estimation may make it easier to link or compare sensitive relational data, to match proprietary or copyrighted structured artifacts such as code, or to support surveillance-style analysis of networked data if such systems are used without appropriate governance. These risks are indirect in the present paper, since we do not introduce a sensitive new dataset or a high-risk released model, but they are worth acknowledging as graph matching systems become more accurate.

There is also a resource trade-off. Geometry-aware encoders such as the FSW-based models~\citep{sverdlov2025fsw} studied here can incur additional computational cost relative to standard message-passing baselines. From a societal perspective, this means that improved faithfulness should be weighed against efficiency, accessibility, and environmental cost. One practical implication is that future work should aim not only for stronger guarantees, but also for more efficient bi-Lipschitz architectures and transparent reporting of compute-performance trade-offs.

\section{Limitations}
\label{app:limitations}
Our claims are intentionally calibrated to the regime analyzed in Section~\ref{sec:theory}. The theory is developed on finite 1-WL-separable graph collections and uses the doubly-stochastic surrogate geometry as an intermediate bridge to GED. As a result, the guarantees explain when improved encoder geometry should improve conditioning of neural GED estimators, but they do not imply exact GED recovery, universal constants, or uniform guarantees outside the considered graph family and benchmark setting.

A second limitation concerns the scope of the target similarity notion. The present analysis is centered on the GED setting studied in the paper and on the surrogate geometries that naturally connect to current encoder constructions. Extending the framework to variable edit-cost settings appears plausible in principle, but requires a more careful comparison between the target edit geometry and the surrogate quantities used by the downstream estimator. Related objectives such as maximum common subgraph (MCS), subgraph edit distance, or other partial-matching notions would likewise require task-specific extensions of the theory.

A third limitation is architectural. Our framework directly covers graph-similarity predictors and matching/alignment-based estimators whose key quantities can be expressed in terms of graph-level or node-level embeddings. This includes a broad range of current neural GED pipelines, but it does not yet cleanly capture autoregressive neural-combinatorial approaches such as Gelato~\citep{pellizzoni2025gelato}, where predictions are produced through sequential decision-making rather than a single static surrogate. A useful future direction is to connect encoder geometry to per-step guarantees for such decoders or to partial-alignment objectives arising along the decoding trajectory.

The experiments have complementary limitations. To maintain comparability across a broad baseline suite, we use node attributes when available but do not use edge labels, since several compared baselines do not support them consistently. In addition, we intentionally limit hyperparameter retuning after encoder replacement in order to isolate the effect of encoder geometry. This strengthens the causal interpretation of the comparison, but it may understate the best performance achievable by some geometry-aware variants after full method-specific tuning.

There are also efficiency limitations. FSW-GNN~\citep{sverdlov2025fsw} provides an explicit and analyzable realization of the bi-Lipschitz perspective, but its aggregation mechanism is computationally more expensive than standard sum-based message passing. Accordingly, the present paper should be read as establishing a geometry-performance trade-off, not as claiming that the current FSW instantiation is always the most efficient choice. Future work should therefore expand the design space of bi-Lipschitz graph encoders, improve their computational efficiency, and determine which geometric guarantees are most important empirically.

Finally, the current theory remains tied to the 1-WL expressive regime. This already covers many of the benchmarks and encoder families used in contemporary neural GED work, but it is not the end of the story. Extending the analysis to higher-WL or otherwise more expressive graph representation regimes is an important open direction for understanding when stronger geometric guarantees can continue to improve neural graph matching beyond the current setting~\citep{moscatelli20253wl}.

\section{Related Neural GED Methods}
\label{app:more_related_work}

Neural GED methods can be broadly grouped into \emph{graph-level prediction pipelines}, \emph{matching-based estimators}, and a complementary line of work focused on \emph{edit-cost learning}. Across these families, prior work has introduced strong architectural innovations, but most existing theory remains tied to specific components -- such as prediction heads, matching routines, or surrogate objectives -- rather than the geometry of the encoder itself, as we discuss in more detail below.

\subsection{Graph-level prediction pipelines}
Possibly the largest class of neural graph matching methods treat GED estimation primarily as a graph-pair prediction problem. Early methods such as SimGNN~\citep{bai2020simgnn} showed that accurate GED approximation is possible without explicit combinatorial search by combining graph-level embeddings with histogram features and a neural tensor network~\citep{ntn}. This design was influential because it showed that one can learn accurate GED surrogates without explicit combinatorial search, but its justification is largely architectural and empirical -- it did not provide a theory relating the learned graph representation to a graph metric or to GED through explicit distortion bounds. Subsequent models such as GraphSim~\citep{bai2020learning}, EGSC~\citep{egsc}, ERIC~\citep{eric}, TaGSim~\citep{tagsim}, and H2MN~\citep{h2mn} introduced stronger fusion modules, CNN-based similarity learning, alignment-aware regularization, edit-type decomposition, and higher-order interaction modeling. These methods often improve empirical performance through better architectural design, but their analyses are also largely empirical or specific to individual modules.

Graph Matching Networks (GMN)~\citep{li2019graph} moved further toward pair-dependent reasoning by introducing cross-graph attention during message passing. Architecturally, this is more expressive than independent graph embeddings because the representation of each graph depends on the paired graph itself. However, the corresponding theory is again tied to the matching architecture rather than to a quantitative notion of encoder faithfulness to edit geometry. In particular, GMN gives a mechanism for improved pairwise comparison, but not a graph-level or node-level distortion theory explaining when Euclidean or matching costs computed from learned embeddings should serve as well-conditioned GED surrogates.

GREED~\citep{ranjan2022greed} is the closest prior theoretical work among graph-level methods. It uses a pair-independent Siamese encoder with a constrained prediction head that ensures metric-like GED outputs and triangle inequality guarantees for SED. This is a meaningful output-space guarantee, but it operates at the prediction layer and does not analyze whether the encoder itself preserves graph geometry. Our perspective is complementary: we study how encoder geometry influences the quality of graph-level GED surrogates before any prediction head is applied.

GraSP~\citep{zheng2025grasp} occupies an important intermediate point in this family. It avoids expensive cross-graph interactions while using positional encodings, residual graph backbones, and multi-scale pooling to produce graph embeddings for GED and MCS prediction. Its main theoretical contribution is an expressiveness guarantee showing that the architecture can surpass the \(1\)-WL test. While important, this addresses graph distinguishability rather than whether embedding distances quantitatively preserve edit-relevant geometry.

\subsection{Matching and alignment based estimators}
A second class of methods estimates GED through explicit alignment mechanisms that more closely resemble the edit process itself. Methods such as ISONET~\citep{isonet} formulate graph comparison through differentiable matching procedures, often using assignment relaxations or Sinkhorn-style alignment mechanisms to produce interpretable correspondences. Recent fused optimal transport formulations such as FGWAlign~\citep{tang2025fused} similarly combine structural and feature discrepancies in a unified transport objective.

GEDGNN~\citep{piao2023computing} is representative of this family. It learns node embeddings for the graph pair and then builds two cross-matrix modules: one predicts a matching matrix and the other predicts a cost matrix. The final GED estimate is obtained by weighting costs using the learned matching scores, followed by post-processing with \(k\)-best matching to recover an edit path. This is a strong method-design contribution and improves interpretability relative to pure regression. Its theoretical analysis is correspondingly tied to the matching-matrix extraction procedure and the post-processing stage. However, what it does not provide is an account of how the geometry of the node encoder propagates into the relaxed matching costs and hence into the optimized alignment objective. Our matching-based theory is designed to fill precisely this gap.

GraphEDX~\citep{jain2024graph} takes a more optimization-centric route. It pads graphs to equal size, uses contextual node embeddings, learns a soft node permutation through a differentiable Gumbel-Sinkhorn alignment generator, derives node-pair alignments, and then replaces the original quadratic assignment objective by differentiable neural set-divergence surrogates. This is one of the most principled formulations in recent neural GED work, particularly for general-cost settings. Still, its theory is centered on the surrogate objective and the alignment construction, not on the quantitative geometry of the encoder. In the language of our theory, GraphEDX specifies a sophisticated matching objective; our contribution is to analyze how replacing the encoder by one with stronger geometric guarantees can improve the faithfulness of objectives of exactly this kind.

Recent optimal transport and graph-matching formulations, including fused Gromov--Wasserstein approaches such as FGWAlign~\citep{tang2025fused} or GWD~\citep{xugromov}, also move the field in a more principled direction. They incorporate structural and feature discrepancies into a unified alignment objective and often operate in the equalized-node setting used in practice for GED approximation. However, even in these formulations, the representation map supplying node or graph features is typically not analyzed through explicit graph-level or node-level distortion guarantees. Our theory is intended to be compatible with such methods rather than to replace them: it provides a geometric lens through which one can reason about why some encoders may make these relaxed alignment objectives better behaved.

\subsection{Learning edit costs and alternative paradigms}
Another active direction modifies the edit-cost model itself. TaGSim~\citep{tagsim} and related works emphasize fine-grained or type-aware cost design, often reducing the burden of exact GED supervision by generating synthetic training pairs or decomposing GED into edit-type-aware targets. These approaches can substantially improve practical accuracy, especially when labels and edit types matter, but the guiding principle is typically cost engineering rather than a metric theory of representation quality.

More recent work such as GEDAN~\citep{leonardi2026gedan} continues this direction by learning edit costs more explicitly. This is orthogonal to our contribution. Learning better costs may improve a fixed neural GED pipeline, while our question is different: given any downstream cost model or alignment mechanism, when does the encoder preserve the relevant graph or node geometry well enough that the learned quantities remain faithful surrogates for GED?

At the opposite end of the design spectrum, GRAIL~\citep{verma2025grail} departs from neural embedding pipelines altogether and uses large language models to synthesize programs for GED approximation and node alignment. This offers a compelling alternative centered on interpretability and cross-domain generalization. Because it is not an encoder-based neural surrogate in the usual sense, it lies outside the scope of the present theory. Its inclusion is nevertheless useful for perspective: not all progress in GED approximation comes from neural representation learning, which makes it even more important to articulate clearly what a representation-theoretic contribution such as ours does and does not claim.

\subsection{Scope of our contribution}

Taken together, prior work provides several forms of partial theory, among which we highlight: the expressiveness via position encodings in GraSP~\citep{zheng2025grasp}, output-space metric guarantees in GREED~\citep{ranjan2022greed}, matching-specific analyses in GEDGNN~\citep{piao2023computing}, and optimization-based surrogate guarantees in GraphEDX~\citep{jain2024graph} and transport-based methods~\citep{tang2025fused, xugromov}. These are all meaningful advances, but they operate at different stages of the pipeline. Our contribution addresses a different, more structural question: \emph{how does encoder geometry affect neural GED estimation across these architectural families?} We provide an architecture-class-level theory for both\textit{ graph-level predictors} and \textit{matching-based estimators} by explicitly characterizing how graph-level and node-level distortion influence the quality of GED surrogates. This makes our framework modular and complementary to existing architectures rather than tied to any single method. By analyzing the geometry of the encoder itself, we provide insights that can inform the design of future architectures across both families.

Moreover, recent surveys and evaluation papers have argued that the field needs clearer conceptual organization and more reliable empirical assessment~\citep{liu2024graph,xu2025graph,roy2025position}. That broader perspective further motivates a theory that is modular, architecture-aware, and explicit about its assumptions. In that sense, our work aims to contribute not only a geometric explanation for empirical gains, but also a cleaner theoretical vocabulary for discussing what different neural GED methods actually learn and where their inductive biases enter.

\section{FSW-GNN, the DS pseudometric, and implications for neural GED}
\label{app:fsw_ds_ged}

\subsection{DS pseudometric, WL equivalence, and soft alignment}
\label{app:fsw_ds}

Let \(G=(A,X)\) and \(H=(B,Y)\) be attributed graphs with the same number \(n\) of nodes, where
\(A,B \in \mathbb{R}^{n\times n}\) are adjacency matrices and
\(X=(x_1,\dots,x_n)^\top\), \(Y=(y_1,\dots,y_n)^\top\) are node-feature matrices with
\(x_i,y_j \in \Omega \subset \mathbb{R}^d\).

The graph edit distance (GED) between \(G\) and \(H\), denoted \(d_{\mathrm{GED}}(G,H)\), is the minimum total cost of an edit path that transforms \(G\) into \(H\) using admissible node and edge insertions, deletions, and substitutions. In this paper, we work with the standard uniform-cost setting, so \(d_{\mathrm{GED}}(G,H)=0\) if and only if \(G\) and \(H\) are isomorphic, and larger values correspond to greater structural dissimilarity. The quantity \(d_{\mathrm{GED}}\) is the discrete target metric of interest, whereas the DS pseudometric introduced below serves as a continuous soft-alignment surrogate that is more amenable to geometric analysis.

Let \(\mathcal{D}_n\) denote the Birkhoff polytope of doubly-stochastic matrices.
In the equal-size setting, the attributed DS pseudometric used in the FSW-GNN analysis can be written as
\[
d_{\mathrm{DS}}(G,H)
:=
\min_{S\in\mathcal{D}_n}
\Bigl(
\|AS-SB\|_1 + \langle S,C_{X,Y}\rangle
\Bigr),
\qquad
(C_{X,Y})_{ij}:=\|x_i-y_j\|_1.
\]
The first term measures failure of equivariance of the adjacency operators under a soft correspondence \(S\), while the second term measures feature mismatch under the same correspondence. The metric is therefore defined on the same alignment space that underlies doubly-stochastic relaxations of graph matching.

For featureless equal-size graphs, the structural part of \(d_{\mathrm{DS}}\) reduces to minimizing the residual of the relaxed commutation relation
\[
AS=SB,
\qquad
S\in\mathcal{D}_n.
\]
This is the doubly-stochastic relaxation of exact permutation alignment, obtained by replacing permutation matrices by their convex hull. In particular, \(d_{\mathrm{DS}}\) belongs to the same relaxation family as fractional graph matching and fractional isomorphism, which is why it is closely connected to \(1\)-WL equivalence.

\begin{lemma}[Zero set of \(d_{\mathrm{DS}}\)~\citep{sverdlov2025fsw}]
\label{lem:ds_zero_set}
For graphs in the FSW-GNN setting, \(d_{\mathrm{DS}}(G,H)=0\) if and only if \(G\) and \(H\) are \(1\)-WL equivalent.
\end{lemma}

Concretely, the objective defining \(d_{\mathrm{DS}}\) is nonnegative, so \(d_{\mathrm{DS}}(G,H)=0\) holds precisely when there exists a feasible doubly-stochastic alignment with zero structural and feature mismatch. The FSW-GNN analysis shows that this vanishing condition coincides with \(1\)-WL equivalence on the relevant graph class.

The previous lemma is the key structural reason we use \(d_{\mathrm{DS}}\) as the intermediary geometry in this work. The metric is weaker than graph isomorphism, as any WL-type metric must be, but it is already strong enough to separate graphs throughout the \(1\)-WL-separable regime in which the present analysis is carried out.

A second useful observation is that \(d_{\mathrm{DS}}\) is a genuine convex relaxation of a permutation-restricted alignment cost.

\begin{proposition}[Permutation relaxation~\citep{}]
\label{prop:ds_perm_relax}
Let \(\Pi_n\) denote the set of \(n\times n\) permutation matrices, and define
\[
d_{\Pi}(G,H)
:=
\min_{P\in\Pi_n}
\Bigl(
\|AP-PB\|_1 + \langle P,C_{X,Y}\rangle
\Bigr).
\]
Then
\[
d_{\mathrm{DS}}(G,H) \le d_{\Pi}(G,H)
\qquad
\text{for all equal-size attributed graphs } G,H.
\]
\end{proposition}
Proposition~\ref{prop:ds_perm_relax} is trivially true, because every permutation matrix is doubly-stochastic, \(\Pi_n \subseteq \mathcal{D}_n\), and the claim follows immediately by minimizing the same objective over the larger feasible set \(\mathcal{D}_n\). This proposition isolates the role of \(d_{\mathrm{DS}}\): it is not a discrete edit metric, but it is the soft-alignment relaxation of a permutation-based structural discrepancy. This is precisely the level at which many neural graph matching and neural GED methods operate before producing a final score or projecting to a hard alignment.

\subsection{Mechanism behind the FSW--DS theorem}
\label{app:fsw_ds_mechanism}

The graph-level FSW-GNN theorem uses a stronger geometric mechanism. The first ingredient is the FSW multiset embedding, which provides a permutation-invariant representation with quantitative metric control at the multiset level~\citep{sverdlov2025fsw}. The second ingredient is that, for fixed graph size and bounded feature domain, the message-passing, update, and readout maps of FSW-GNN are piecewise linear on the cells determined by the relevant sorting patterns~\citep{sverdlov2025fsw}. The third ingredient is that the DS objective is itself piecewise linear in the same bounded regime when formulated with the \(\ell_1\)-norm~\citep{sverdlov2025fsw}.

These ingredients imply that both
\[
(G,H)\mapsto \|\Phi_{\mathrm{FSW}}(G)-\Phi_{\mathrm{FSW}}(H)\|_2
\qquad\text{and}\qquad
(G,H)\mapsto d_{\mathrm{DS}}(G,H)
\]
are nonnegative piecewise-linear functions on a compact domain, and Lemma~\ref{lem:ds_zero_set} identifies their common zero set through the WL-equivalence characterization in\ \citet{sverdlov2025fsw}. The bi-Lipschitz bound then follows by comparing two such functions on each cell and taking global extrema over the finitely many cells. In other words, the proof is fundamentally a quantitative comparison of two compatible surrogate geometries. A WL-equivalent encoder may separate the same graph classes as another encoder while still distorting distances far more severely. The FSW-GNN theorem rules out that failure mode with respect to \(d_{\mathrm{DS}}\), whereas the available theory for standard sum-based MPNNs is qualitatively weaker and provides only lower-H\"older guarantees in expectation~\citep{davidsonHolderStabilityMultiset2024}. For neural GED, where the downstream model consumes distances, costs, or alignments rather than just class labels, that distinction proves substantial.

The recent sorting-based follow-up work~\citep{dym2026quantitative} further supports this viewpoint by studying quantitative distortion bounds for sorting-based permutation-invariant embeddings, which is exactly the multiset-embedding mechanism underlying the FSW line. This makes the FSW perspective especially suitable when one wants metric control rather than expressivity alone.

\subsection{Additional remarks on the DS--GED connection}
\label{app:ds_ged}

After node-count equalization, both neural GED objectives and \(d_{\mathrm{DS}}\) are formulated over cross-graph correspondences. The distinction is that GED remains a discrete optimization over edit paths or node permutations, whereas \(d_{\mathrm{DS}}\) passes to a fractional correspondence space. This is why \(d_{\mathrm{DS}}\) should be viewed as a continuous surrogate geometry for edit distance rather than as an alternative definition of GED itself. The equalized-node regime is especially natural here. Practical neural GED systems commonly pad or augment graph pairs so that both graphs have the same number of nodes before constructing matching matrices, Sinkhorn plans, or transport objectives. In that regime, the feasible set \(\mathcal{D}_n\) used by \(d_{\mathrm{DS}}\) is not an artificial choice imposed only for analysis, but rather the same soft alignment class that already appears, explicitly or implicitly, in assignment-based, Sinkhorn-based, and transport-based neural GED architectures~\citep{roy2025position,liu2024graph,piao2023computing,jain2024graph,tang2025fused}.

The DS bridge is also technically preferable to the Tree Mover's Distance~\citep{chuang, sverdlov2025fsw} bridge in the present setting. The TMD-side bi-Lipschitz theorem for FSW-GNN requires the feature domain to avoid the origin, while equalization by dummy nodes often introduces neutral or zero features. By contrast, the DS formulation is compatible with exactly that padded-node regime, so it matches the implementation conventions of current neural GED pipelines more closely.

A second reason is that the comparison between \(d_{\mathrm{DS}}\) and \(d_{\mathrm{GED}}\) is naturally graph-dataset-dependent. On a fixed dataset or graph family, the constants relating the two metrics absorb graph size range, feature alphabet, or feature geometry, and the chosen equalization rule. They are therefore not expected to be universal. This is also the sense in which our theory isolates the encoder-side contribution. The dataset-dependent constants governing the comparison between \(d_{\mathrm{DS}}\) and \(d_{\mathrm{GED}}\) are fixed by the graph family and equalization scheme, not by the encoder. What FSW-style encoders control instead is the additional distortion incurred when this surrogate geometry is mapped into Euclidean graph or node representations, through the bi-Lipschitz constants in our graph-level and node-level bounds.

This also clarifies why improvements from FSW-style encoders should be expected to be systematic but not perfect. Even when \(\Phi_{\mathrm{FSW}}\) is bi-Lipschitz with respect to \(d_{\mathrm{DS}}\), the final GED estimate still depends on the residual gap between \(d_{\mathrm{DS}}\) and \(d_{\mathrm{GED}}\), on the approximation quality of the downstream prediction head or alignment module, and on optimization error in learning. Thus the expected gain is better conditioning of the surrogate geometry used by the method, not exact recovery of GED in all cases.

Taken together, these observations place \(d_{\mathrm{DS}}\) in the right intermediate position for the present analysis. It is close enough to the alignment objectives used by neural GED architectures to be operationally relevant, yet regular enough to admit a graph-level bi-Lipschitz embedding theorem for FSW-GNN. This combination is what makes the FSW--DS connection useful for analyzing geometry-informed improvements in neural GED.

\section{Proofs for Section~\ref{sec:theory}}
\label{app:proofs}

\setcounter{proposition}{0}
\setcounter{theorem}{0}
\setcounter{corollary}{0}
\setcounter{lemma}{0}
\renewcommand{\theproposition}{E.\arabic{proposition}}
\renewcommand{\thetheorem}{E.\arabic{theorem}}
\renewcommand{\thecorollary}{E.\arabic{corollary}}
\renewcommand{\thelemma}{E.\arabic{lemma}}

For convenience, we restate the results \textit{with appendix numbering}.

\subsection{Proofs for Graph Similarity Predictors}

\begin{proposition}[GED--DS comparability]
\label{prop:app_ged_ds_dataset}
Define
\[
a_{\mathcal G}
:=
\min_{\substack{G,H\in\mathcal G\\G\neq H}}
\frac{d_{\mathrm{GED}}(G,H)}{d_{\mathrm{DS}}(G,H)},
\qquad
b_{\mathcal G}
:=
\max_{\substack{G,H\in\mathcal G\\G\neq H}}
\frac{d_{\mathrm{GED}}(G,H)}{d_{\mathrm{DS}}(G,H)}.
\]
Then
\[
0<a_{\mathcal G}\le b_{\mathcal G}<\infty,
\]
and, for all \(G,H\in\mathcal G\),
\[
a_{\mathcal G}\, d_{\mathrm{DS}}(G,H)
\;\le\;
d_{\mathrm{GED}}(G,H)
\;\le\;
b_{\mathcal G}\, d_{\mathrm{DS}}(G,H).
\]
\end{proposition}

\begin{proof}
Because \(\mathcal G\) is finite modulo graph isomorphism, the set of distinct pairs
\[
\{(G,H)\in\mathcal G\times\mathcal G : G\neq H\}
\]
is finite. Since \(\mathcal G\subseteq \mathcal{G}^{\mathrm{WL}}_{\le N}\) is pairwise
\(1\)-WL-separable and \(d_{\mathrm{DS}}\) is \(1\)-WL-equivalent, we have
\[
d_{\mathrm{DS}}(G,H)>0
\qquad
\forall\, G\neq H \in \mathcal G.
\]
Likewise, because we work modulo graph isomorphism and use uniform nonnegative edit
costs, we also have
\[
d_{\mathrm{GED}}(G,H)>0
\qquad
\forall\, G\neq H \in \mathcal G.
\]
Hence the ratio
\[
\frac{d_{\mathrm{GED}}(G,H)}{d_{\mathrm{DS}}(G,H)}
\]
is positive and finite on a finite set, so its minimum and maximum are well-defined and
satisfy
\[
0<a_{\mathcal G}\le b_{\mathcal G}<\infty.
\]
The displayed inequality then follows immediately for all distinct pairs, while for
\(G=H\) both sides equal zero.
\end{proof}

\begin{proposition}[Graph-level bounds]
\label{prop:app_graph_sandwich}
For every \(G,H \in \mathcal G\),
\[
\frac{a_{\mathcal G}}{C_G}\,\|\Phi(G)-\Phi(H)\|_2
\;\le\;
d_{\mathrm{GED}}(G,H)
\;\le\;
\frac{b_{\mathcal G}}{c_G}\,\|\Phi(G)-\Phi(H)\|_2.
\]
\end{proposition}

\begin{proof}
By the upper inequality in Assumption~\ref{ass:graph_bilip},
\[
\|\Phi(G)-\Phi(H)\|_2 \le C_G\, d_{\mathrm{DS}}(G,H),
\]
hence
\[
d_{\mathrm{DS}}(G,H) \ge \frac{1}{C_G}\,\|\Phi(G)-\Phi(H)\|_2.
\]
Combining this with the lower inequality in
Proposition~\ref{prop:app_ged_ds_dataset} yields
\[
d_{\mathrm{GED}}(G,H)
\ge
a_{\mathcal G}\, d_{\mathrm{DS}}(G,H)
\ge
\frac{a_{\mathcal G}}{C_G}\,\|\Phi(G)-\Phi(H)\|_2.
\]

Similarly, by the lower inequality in Assumption~\ref{ass:graph_bilip},
\[
\|\Phi(G)-\Phi(H)\|_2 \ge c_G\, d_{\mathrm{DS}}(G,H),
\]
so
\[
d_{\mathrm{DS}}(G,H) \le \frac{1}{c_G}\,\|\Phi(G)-\Phi(H)\|_2.
\]
Combining this with the upper inequality in
Proposition~\ref{prop:app_ged_ds_dataset} gives
\[
d_{\mathrm{GED}}(G,H)
\le
b_{\mathcal G}\, d_{\mathrm{DS}}(G,H)
\le
\frac{b_{\mathcal G}}{c_G}\,\|\Phi(G)-\Phi(H)\|_2.
\]
This proves the claim.
\end{proof}

\begin{theorem}[Best scalar calibration for norm-based predictors]
\label{thm:app_calibration}
Let
\[
\alpha := \frac{a_{\mathcal G}}{C_G},
\qquad
\beta := \frac{b_{\mathcal G}}{c_G},
\qquad
\gamma^\star := \frac{\alpha+\beta}{2}.
\]
Then, for every \(G,H \in \mathcal G\),
\[
\left|
\gamma^\star \|\Phi(G)-\Phi(H)\|_2 - d_{\mathrm{GED}}(G,H)
\right|
\le
\frac{\beta-\alpha}{2}\,\|\Phi(G)-\Phi(H)\|_2.
\]
\end{theorem}

\begin{proof}
By Proposition~\ref{prop:app_graph_sandwich},
\[
d_{\mathrm{GED}}(G,H)
\in
\left[
\alpha \|\Phi(G)-\Phi(H)\|_2,\,
\beta \|\Phi(G)-\Phi(H)\|_2
\right].
\]
Let \(x := \|\Phi(G)-\Phi(H)\|_2\). Then
\[
d_{\mathrm{GED}}(G,H)\in[\alpha x,\beta x].
\]
The midpoint of this interval is \(\gamma^\star x\), and the radius is
\[
\frac{\beta-\alpha}{2}x.
\]
Therefore
\[
\left| \gamma^\star x - d_{\mathrm{GED}}(G,H) \right|
\le
\frac{\beta-\alpha}{2}x.
\]
Substituting back \(x=\|\Phi(G)-\Phi(H)\|_2\) gives the result.
\end{proof}

\begin{corollary}[Ranking preservation for monotone norm-head predictors]
\label{cor:app_ranking_graph}
Let \(\widehat d_\psi(G,H):=\psi\!\left(\|\Phi(G)-\Phi(H)\|_2\right)\), where \(\psi:\mathbb{R}_{\ge 0}\to\mathbb{R}\) is strictly increasing. Define
\[
\kappa_G := \frac{\beta}{\alpha}
= \frac{b_{\mathcal G} C_G}{a_{\mathcal G} c_G}.
\]
Then for any two graph pairs \((G_1,H_1)\) and \((G_2,H_2)\) (and similarly with $<$ and $\kappa_G^{-1}$),
\[
d_{\mathrm{GED}}(G_1,H_1)
>
\kappa_G\, d_{\mathrm{GED}}(G_2,H_2)
\quad\Longrightarrow\quad
\widehat d_\psi(G_1,H_1)>\widehat d_\psi(G_2,H_2).
\]
\end{corollary}

\begin{proof}
Let
\[
x_i := \|\Phi(G_i)-\Phi(H_i)\|_2,
\qquad
d_i := d_{\mathrm{GED}}(G_i,H_i),
\qquad
i\in\{1,2\}.
\]
By Proposition~\ref{prop:app_graph_sandwich},
\[
\alpha x_i \le d_i \le \beta x_i,
\qquad i\in\{1,2\}.
\]
Hence
\[
x_1 \ge \frac{d_1}{\beta},
\qquad
x_2 \le \frac{d_2}{\alpha}.
\]
If \(d_1>\frac{\beta}{\alpha}d_2\), then
\[
x_1 \ge \frac{d_1}{\beta}
>
\frac{d_2}{\alpha}
\ge x_2.
\]
Since \(\psi\) is strictly increasing, it follows that
\[
\widehat d_\psi(G_1,H_1)=\psi(x_1)>\psi(x_2)=\widehat d_\psi(G_2,H_2).
\]
The second implication can be obtained by swapping the roles of the two pairs.
\end{proof}

\begin{proposition}[Stability of Lipschitz interaction heads]
\label{prop:app_lipschitz_head}
Let \(h:\mathbb{R}^m\times\mathbb{R}^m\to\mathbb{R}\) be \(L_h\)-Lipschitz with respect to \(\|(u,v)\|_{\oplus} := \|u\|_2+\|v\|_2\). Define \(s_h(G,H) := h(\Phi(G),\Phi(H))\). Then, for every \(G,H,G',H' \in \mathcal G\),
\[
|s_h(G,H)-s_h(G',H')|
\le
L_h C_G\,
\bigl(
d_{\mathrm{DS}}(G,G')+d_{\mathrm{DS}}(H,H')
\bigr).
\]
\end{proposition}

\begin{proof}
By the Lipschitz property of \(h\),
\[
|s_h(G,H)-s_h(G',H')|
\le
L_h\Bigl(
\|\Phi(G)-\Phi(G')\|_2+\|\Phi(H)-\Phi(H')\|_2
\Bigr).
\]
Applying the upper inequality in Assumption~\ref{ass:graph_bilip} to each term gives
\[
\|\Phi(G)-\Phi(G')\|_2 \le C_G\, d_{\mathrm{DS}}(G,G'),
\qquad
\|\Phi(H)-\Phi(H')\|_2 \le C_G\, d_{\mathrm{DS}}(H,H').
\]
Substituting these inequalities yields
\[
|s_h(G,H)-s_h(G',H')|
\le
L_h C_G\,
\bigl(
d_{\mathrm{DS}}(G,G')+d_{\mathrm{DS}}(H,H')
\bigr),
\]
as claimed.
\end{proof}

\subsection{Proofs for Matching-based Estimators}

Recall that, we have fixed an equalization rule that maps each pair
\((G,H)\in\mathcal G\times\mathcal G\) to an equalized pair
\((\bar G,\bar H)\) with
\[
n_{\bar G,\bar H} := |V(\bar G)| = |V(\bar H)|.
\]
The corresponding feasible set of alignment plans is
\[
\mathcal{D}_{\bar G,\bar H}
:=
\left\{
P\in\mathbb{R}_{\ge 0}^{n_{\bar G,\bar H}\times n_{\bar G,\bar H}}
:\;
P\mathbf{1}=\mathbf{1},
\;
P^\top\mathbf{1}=\mathbf{1}
\right\}.
\]
For each such pair, define
\[
C^\delta_{\bar G,\bar H}(i,j):=\delta_{\bar G,\bar H}(i,j),
\qquad
C^\phi_{\bar G,\bar H}(i,j):=\|\phi_{\bar G}(i)-\phi_{\bar H}(j)\|_2,
\]
and for \(P\in\mathcal D_{\bar G,\bar H}\),
\[
L_\delta^{\bar G,\bar H}(P)
:=
\langle P, C^\delta_{\bar G,\bar H}\rangle
=
\sum_{i\in V(\bar G)}\sum_{j\in V(\bar H)}
P_{ij}\,\delta_{\bar G,\bar H}(i,j),
\]
\[
L_\phi^{\bar G,\bar H}(P)
:=
\langle P, C^\phi_{\bar G,\bar H}\rangle
=
\sum_{i\in V(\bar G)}\sum_{j\in V(\bar H)}
P_{ij}\,\|\phi_{\bar G}(i)-\phi_{\bar H}(j)\|_2.
\]

\begin{lemma}[Plan-wise distortion of alignment costs]
\label{lem:app_planwise}
Under Assumption~\ref{ass:node_bilip}, for every pair
\((G,H)\in\mathcal G\times\mathcal G\) with equalized versions
\((\bar G,\bar H)\), and every \(P\in\mathcal{D}_{\bar G,\bar H}\),
\[
c_V\, L_\delta^{\bar G,\bar H}(P)
\;\le\;
L_\phi^{\bar G,\bar H}(P)
\;\le\;
C_V\, L_\delta^{\bar G,\bar H}(P).
\]
\end{lemma}

\begin{proof}
Fix a pair \((G,H)\in\mathcal G\times\mathcal G\), let
\((\bar G,\bar H)\) be its equalized versions under the fixed equalization rule, and
fix \(P\in\mathcal D_{\bar G,\bar H}\).

By Assumption~\ref{ass:node_bilip}, for every \(i\in V(\bar G)\) and
\(j\in V(\bar H)\),
\[
c_V\, \delta_{\bar G,\bar H}(i,j)
\le
\|\phi_{\bar G}(i)-\phi_{\bar H}(j)\|_2
\le
C_V\, \delta_{\bar G,\bar H}(i,j).
\]
Since \(P_{ij}\ge 0\), multiplication by \(P_{ij}\) preserves the inequalities:
\[
c_V\, P_{ij}\,\delta_{\bar G,\bar H}(i,j)
\le
P_{ij}\,\|\phi_{\bar G}(i)-\phi_{\bar H}(j)\|_2
\le
C_V\, P_{ij}\,\delta_{\bar G,\bar H}(i,j).
\]
Summing over all \(i\in V(\bar G)\) and \(j\in V(\bar H)\) yields
\[
c_V \sum_{i\in V(\bar G)}\sum_{j\in V(\bar H)}
P_{ij}\,\delta_{\bar G,\bar H}(i,j)
\le
\sum_{i\in V(\bar G)}\sum_{j\in V(\bar H)}
P_{ij}\,\|\phi_{\bar G}(i)-\phi_{\bar H}(j)\|_2
\le
C_V \sum_{i\in V(\bar G)}\sum_{j\in V(\bar H)}
P_{ij}\,\delta_{\bar G,\bar H}(i,j).
\]
By the definitions of \(L_\delta^{\bar G,\bar H}(P)\) and
\(L_\phi^{\bar G,\bar H}(P)\), this is exactly
\[
c_V\, L_\delta^{\bar G,\bar H}(P)
\le
L_\phi^{\bar G,\bar H}(P)
\le
C_V\, L_\delta^{\bar G,\bar H}(P).
\]
\end{proof}

\begin{theorem}[Distortion bound for alignment objectives]
\label{thm:app_alignment}
For each pair \((G,H)\in\mathcal G\times\mathcal G\), let \(\mathcal{R}_{\bar G,\bar H}:\mathcal{D}_{\bar G,\bar H}\to\mathbb{R}\) be any structural term independent of the encoder, and let \(\lambda\ge 0\). Define
\[
J_\delta^{\bar G,\bar H}(P)
:=
\mathcal{R}_{\bar G,\bar H}(P)+\lambda L_\delta^{\bar G,\bar H}(P), \quad
J_\phi^{\bar G,\bar H}(P)
:=
\mathcal{R}_{\bar G,\bar H}(P)+\lambda L_\phi^{\bar G,\bar H}(P).
\]
Then, for every pair \((G,H)\in\mathcal G\times\mathcal G\),
\[
\min_{P\in\mathcal{D}_{\bar G,\bar H}}
\Bigl(
\mathcal{R}_{\bar G,\bar H}(P)+\lambda c_V L_\delta^{\bar G,\bar H}(P)
\Bigr)
\;\le\;
\min_{P\in\mathcal{D}_{\bar G,\bar H}} J_\phi^{\bar G,\bar H}(P)
\;\le\;
\min_{P\in\mathcal{D}_{\bar G,\bar H}}
\Bigl(
\mathcal{R}_{\bar G,\bar H}(P)+\lambda C_V L_\delta^{\bar G,\bar H}(P)
\Bigr).
\]
\end{theorem}

\begin{proof}
Fix a pair \((G,H)\in\mathcal G\times\mathcal G\), and let
\((\bar G,\bar H)\) be its equalized versions. By
Lemma~\ref{lem:app_planwise}, for every \(P\in\mathcal D_{\bar G,\bar H}\),
\[
c_V\, L_\delta^{\bar G,\bar H}(P)
\le
L_\phi^{\bar G,\bar H}(P)
\le
C_V\, L_\delta^{\bar G,\bar H}(P).
\]
Since \(\lambda\ge 0\), multiplication by \(\lambda\) preserves the inequalities:
\[
\lambda c_V\, L_\delta^{\bar G,\bar H}(P)
\le
\lambda L_\phi^{\bar G,\bar H}(P)
\le
\lambda C_V\, L_\delta^{\bar G,\bar H}(P).
\]
Adding the same quantity \(\mathcal R_{\bar G,\bar H}(P)\) to all three terms gives
\[
\mathcal R_{\bar G,\bar H}(P)+\lambda c_V\, L_\delta^{\bar G,\bar H}(P)
\le
\mathcal R_{\bar G,\bar H}(P)+\lambda L_\phi^{\bar G,\bar H}(P)
\le
\mathcal R_{\bar G,\bar H}(P)+\lambda C_V\, L_\delta^{\bar G,\bar H}(P).
\]
That is,
\[
\mathcal R_{\bar G,\bar H}(P)+\lambda c_V\, L_\delta^{\bar G,\bar H}(P)
\le
J_\phi^{\bar G,\bar H}(P)
\le
\mathcal R_{\bar G,\bar H}(P)+\lambda C_V\, L_\delta^{\bar G,\bar H}(P).
\]
Taking the minimum over \(P\in\mathcal D_{\bar G,\bar H}\) in all three expressions
preserves the inequalities and yields the claim.
\end{proof}

\begin{corollary}[Comparability transfer for alignment surrogates]
\label{cor:app_alignment_transfer}
Assume, in addition, that, for the fixed equalization rule and reference discrepancy above, there exist constants \(0<\eta_{1,\mathcal G}\le \eta_{2,\mathcal G}<\infty\), such that, for every pair \((G,H)\in\mathcal G\times\mathcal G\),
\[
\eta_{1,\mathcal G}\, d_{\mathrm{GED}}(G,H)
\;\le\;
\min_{P\in\mathcal{D}_{\bar G,\bar H}} J_\delta^{\bar G,\bar H}(P)
\;\le\;
\eta_{2,\mathcal G}\, d_{\mathrm{GED}}(G,H).
\]
Moreover, assume that \(\mathcal{R}_{\bar G,\bar H}(P)\ge 0\) for all
\((G,H)\in\mathcal G\times\mathcal G\) and all \(P\in\mathcal{D}_{\bar G,\bar H}\).
Let \(\underline c := \min\{1,c_V\}, \quad \overline C := \max\{1,C_V\}\). Then, for every pair \((G,H)\in\mathcal G\times\mathcal G\),
\[
\underline c\,\eta_{1,\mathcal G}\, d_{\mathrm{GED}}(G,H)
\;\le\;
\min_{P\in\mathcal{D}_{\bar G,\bar H}} J_\phi^{\bar G,\bar H}(P)
\;\le\;
\overline C\,\eta_{2,\mathcal G}\, d_{\mathrm{GED}}(G,H).
\]
\end{corollary}

\begin{proof}
Fix a pair \((G,H)\in \mathcal G\times \mathcal G\), and write
\[
M_\delta
:=
\min_{P\in\mathcal D_{\bar G,\bar H}} J_\delta^{\bar G,\bar H}(P),
\qquad
M_\phi
:=
\min_{P\in\mathcal D_{\bar G,\bar H}} J_\phi^{\bar G,\bar H}(P).
\]

We first prove the lower bound. Since
\[
\underline c = \min\{1,c_V\},
\]
we have \(\underline c \le 1\) and \(\underline c \le c_V\). Therefore, for every
\(P\in\mathcal D_{\bar G,\bar H}\),
\[
\mathcal R_{\bar G,\bar H}(P)+\lambda c_V L_\delta^{\bar G,\bar H}(P)
\;\ge\;
\underline c\,\mathcal R_{\bar G,\bar H}(P)
+
\underline c\,\lambda L_\delta^{\bar G,\bar H}(P)
=
\underline c\, J_\delta^{\bar G,\bar H}(P),
\]
where we used \(\mathcal R_{\bar G,\bar H}(P)\ge 0\) and \(L_\delta^{\bar G,\bar H}(P)\ge 0\).
Taking minima over \(P\) gives
\[
\min_{P\in\mathcal D_{\bar G,\bar H}}
\bigl(\mathcal R_{\bar G,\bar H}(P)+\lambda c_V L_\delta^{\bar G,\bar H}(P)\bigr)
\;\ge\;
\underline c\, M_\delta.
\]
By Theorem~\ref{thm:app_alignment},
\[
M_\phi
\;\ge\;
\min_{P\in\mathcal D_{\bar G,\bar H}}
\bigl(\mathcal R_{\bar G,\bar H}(P)+\lambda c_V L_\delta^{\bar G,\bar H}(P)\bigr),
\]
hence
\[
M_\phi \ge \underline c\, M_\delta.
\]
Using the assumed lower comparability of \(J_\delta^{\bar G,\bar H}\),
\[
M_\delta \ge \eta_{1,\mathcal G}\, d_{\mathrm{GED}}(G,H),
\]
we conclude that
\[
M_\phi
\;\ge\;
\underline c\,\eta_{1,\mathcal G}\, d_{\mathrm{GED}}(G,H).
\]

For the upper bound, let
\[
\overline C = \max\{1,C_V\}.
\]
Then \(1\le \overline C\) and \(C_V\le \overline C\). Hence, for every
\(P\in\mathcal D_{\bar G,\bar H}\),
\[
\mathcal R_{\bar G,\bar H}(P)+\lambda C_V L_\delta^{\bar G,\bar H}(P)
\;\le\;
\overline C\,\mathcal R_{\bar G,\bar H}(P)
+
\overline C\,\lambda L_\delta^{\bar G,\bar H}(P)
=
\overline C\, J_\delta^{\bar G,\bar H}(P),
\]
again using nonnegativity of \(\mathcal R_{\bar G,\bar H}(P)\) and \(L_\delta^{\bar G,\bar H}(P)\).
Taking minima over \(P\) yields
\[
\min_{P\in\mathcal D_{\bar G,\bar H}}
\bigl(\mathcal R_{\bar G,\bar H}(P)+\lambda C_V L_\delta^{\bar G,\bar H}(P)\bigr)
\;\le\;
\overline C\, M_\delta.
\]
By Theorem~\ref{thm:app_alignment},
\[
M_\phi
\;\le\;
\min_{P\in\mathcal D_{\bar G,\bar H}}
\bigl(\mathcal R_{\bar G,\bar H}(P)+\lambda C_V L_\delta^{\bar G,\bar H}(P)\bigr),
\]
so
\[
M_\phi \le \overline C\, M_\delta.
\]
Using the assumed upper comparability of \(J_\delta^{\bar G,\bar H}\),
\[
M_\delta \le \eta_{2,\mathcal G}\, d_{\mathrm{GED}}(G,H),
\]
we obtain
\[
M_\phi
\;\le\;
\overline C\,\eta_{2,\mathcal G}\, d_{\mathrm{GED}}(G,H).
\]

Combining the two bounds proves the claim.
\end{proof}

\section{Connecting to Existing Neural GED Architectures}
\label{app:connect}

\setcounter{proposition}{0}
\setcounter{theorem}{0}
\setcounter{corollary}{0}
\setcounter{lemma}{0}
\renewcommand{\theproposition}{F.\arabic{proposition}}
\renewcommand{\thetheorem}{F.\arabic{theorem}}
\renewcommand{\thecorollary}{F.\arabic{corollary}}
\renewcommand{\thelemma}{F.\arabic{lemma}}

In this section, we instantiate the general theory of Section~\ref{sec:theory} for some representative neural GED architectures -- chosen for their performance and distinct architectural choices -- and show precisely how the abstract graph-level and alignment-level results apply to these architectures.

\subsection{Graph Similarity Predictors}
\label{app:connect_graph_predictors}

We first consider neural GED estimators that embed each graph
independently into a graph-level vector and then predict GED from the resulting pair of embeddings. This includes pair-independent Siamese architectures, such as GREED~\citep{ranjan2022greed}, and graph-level prediction pipelines, such as GraSP~\citep{zheng2025grasp}. In these models, the relevant component is the graph-level encoder \(\Phi\), together with either a norm-based distance head or a Lipschitz interaction head.

\subsubsection{GREED~\citep{ranjan2022greed}}
\label{app:connect_greed}

GREED~\citep{ranjan2022greed} uses a Siamese GNN with shared parameters to compute
pair-independent graph embeddings, followed by a norm-based prediction
function for GED; in particular, for GED, GREED takes
\[
F_g(Z_{G_1},Z_{G_2}) = \|Z_{G_1}-Z_{G_2}\|_p,
\]
and the authors pick the \(L_2\) norm in experiments. Thus, in our notation (Section~\ref{sec:theory}), GREED is exactly a
graph similarity predictor with graph encoder
\[
\Phi_{\mathrm{GR}}(G) := Z_G
\]
and prediction head
\[
\widehat d_{\mathrm{GR}}(G,H) := \|\Phi_{\mathrm{GR}}(G)-\Phi_{\mathrm{GR}}(H)\|_2.
\]

Consequently, Assumption~\ref{ass:graph_bilip} implies that
Proposition~\ref{prop:graph_sandwich}, Theorem~\ref{thm:calibration}, and Corollary~\ref{cor:ranking_graph} apply directly to GREED.

\begin{proposition}[GREED as a direct instance of our graph-level theory]
\label{prop:greed_connect}
Assume that the GREED graph encoder \(\Phi_{\mathrm{GR}}\) satisfies Assumption~\ref{ass:graph_bilip} on \(\mathcal G\), i.e.,
\[
c_G\, d_{\mathrm{DS}}(G,H)
\;\le\;
\|\Phi_{\mathrm{GR}}(G)-\Phi_{\mathrm{GR}}(H)\|_2
\;\le\;
C_G\, d_{\mathrm{DS}}(G,H)
\qquad
\forall\, G,H\in\mathcal G.
\]
Then, for every \(G,H\in\mathcal G\),
\[
\frac{a_{\mathcal G}}{C_G}
\|\Phi_{\mathrm{GR}}(G)-\Phi_{\mathrm{GR}}(H)\|_2
\;\le\;
d_{\mathrm{GED}}(G,H)
\;\le\;
\frac{b_{\mathcal G}}{c_G}
\|\Phi_{\mathrm{GR}}(G)-\Phi_{\mathrm{GR}}(H)\|_2.
\]
Moreover, with
\[
\alpha_{\mathcal G}:=\frac{a_{\mathcal G}}{C_G},
\quad
\beta_{\mathcal G}:=\frac{b_{\mathcal G}}{c_G},
\quad
\gamma^\star:=\frac{\alpha_{\mathcal G}+\beta_{\mathcal G}}{2},
\]
the calibrated predictor
\[
\widetilde d_{\mathrm{GR}}(G,H)
:=
\gamma^\star \|\Phi_{\mathrm{GR}}(G)-\Phi_{\mathrm{GR}}(H)\|_2
\]
satisfies
\[
\left|
\widetilde d_{\mathrm{GR}}(G,H)-d_{\mathrm{GED}}(G,H)
\right|
\le
\frac{\beta_{\mathcal G}-\alpha_{\mathcal G}}{2}
\|\Phi_{\mathrm{GR}}(G)-\Phi_{\mathrm{GR}}(H)\|_2.
\]
\end{proposition}

\begin{proof}
GREED uses the Euclidean norm of the difference of graph embeddings as its GED predictor. Therefore, it is exactly the predictor \(\widehat d_\psi(G,H)=\|\Phi(G)-\Phi(H)\|_2\) appearing in Theorem~\ref{prop:graph_sandwich} and Theorem~\ref{thm:calibration}, with \(\Phi=\Phi_{\mathrm{GR}}\).
Substituting \(\Phi_{\mathrm{GR}}\) into those results yields the stated inequalities immediately.
\end{proof}

The main point of Proposition~\ref{prop:greed_connect} is that our theory matches the architectural inductive bias of GREED exactly: a pair-independent graph encoder followed by a norm head.
Hence, if one replaces GREED's encoder by a geometry-informed encoder with better graph-level distortion constants \((c_G,C_G)\), then the surrogate and calibration constants improve accordingly. This gives a precise mechanism by which geometry-informed encoders can improve GREED-style GED prediction.

\subsubsection{GraSP~\citep{zheng2025grasp}}
\label{app:connect_grasp}

GraSP~\citep{zheng2025grasp} computes a graph embedding \(z_G\) for each graph, and then predicts GED through a convex combination of a Euclidean distance term and an NTN interaction term. Writing in the notation of Section~\ref{sec:theory},
\[
\Phi_{\mathrm{GS}}(G):=z_G,
\]
and define
\[
d_{\mathrm{emb}}(G,H):=
\|\Phi_{\mathrm{GS}}(G)-\Phi_{\mathrm{GS}}(H)\|_2,
\]
\[
s_{\mathrm{NTN}}(G,H)
:=
\operatorname{MLP}\!\left(
\operatorname{ReLU}\!\left(
\Phi_{\mathrm{GS}}(G)^\top W_I^{[1:t]} \Phi_{\mathrm{GS}}(H)
+
W_C \operatorname{CONCAT}(\Phi_{\mathrm{GS}}(G),\Phi_{\mathrm{GS}}(H))
+
b
\right)\right),
\]
so that, the GraSP predictor can be written as
\[
\widehat d_{\mathrm{GS}}(G,H)
=
\beta\, d_{\mathrm{emb}}(G,H)
+
(1-\beta)\, s_{\mathrm{NTN}}(G,H),
\]
with a learned scalar \(\beta\).

Unlike GREED, GraSP is not a pure norm-head predictor because of the interaction term. However, it still falls within the scope of the graph-level theory through Proposition~\ref{prop:lipschitz_head}, provided the interaction map is Lipschitz on the relevant compact region of embedding space. We also ignore the effect of the positional encoding to ensure that the expressivity of GraSP remains within the 1-WL family.

\begin{proposition}[GraSP as a Lipschitz-head graph predictor]
\label{prop:grasp_connect}
Assume that the GraSP graph encoder \(\Phi_{\mathrm{GS}}\) satisfies
Assumption~\ref{ass:graph_bilip} on \(\mathcal G\), and that the NTN
interaction head \(s_{\mathrm{NTN}}(\cdot,\cdot)\) is
\(L_{\mathrm{NTN}}\)-Lipschitz with respect to
\[
\|(u,v)\|_\oplus := \|u\|_2+\|v\|_2
\]
on the compact set
\[
\Phi_{\mathrm{GS}}(\mathcal G)\times \Phi_{\mathrm{GS}}(\mathcal G).
\]
Then the full GraSP predictor
\[
\widehat d_{\mathrm{GS}}(G,H)
=
\beta\, \|\Phi_{\mathrm{GS}}(G)-\Phi_{\mathrm{GS}}(H)\|_2
+
(1-\beta)\, s_{\mathrm{NTN}}(G,H)
\]
is \(L_{\mathrm{GS}}\)-stable with respect to \(d_{\mathrm{DS}}\), where
\[
L_{\mathrm{GS}}
:=
\beta C_G + (1-\beta)L_{\mathrm{NTN}} C_G.
\]
More precisely, for all \(G,H,G',H'\in\mathcal G\),
\[
\left|
\widehat d_{\mathrm{GS}}(G,H)-\widehat d_{\mathrm{GS}}(G',H')
\right|
\le
L_{\mathrm{GS}}
\bigl(
d_{\mathrm{DS}}(G,G')+d_{\mathrm{DS}}(H,H')
\bigr).
\]
\end{proposition}

\begin{proof}
First, by the reverse triangle inequality,
\[
\left|
\|\Phi_{\mathrm{GS}}(G)-\Phi_{\mathrm{GS}}(H)\|_2
-
\|\Phi_{\mathrm{GS}}(G')-\Phi_{\mathrm{GS}}(H')\|_2
\right|
\le
\|\Phi_{\mathrm{GS}}(G)-\Phi_{\mathrm{GS}}(G')\|_2
+
\|\Phi_{\mathrm{GS}}(H)-\Phi_{\mathrm{GS}}(H')\|_2.
\]
By Assumption~\ref{ass:graph_bilip},
\[
\|\Phi_{\mathrm{GS}}(G)-\Phi_{\mathrm{GS}}(G')\|_2
\le
C_G d_{\mathrm{DS}}(G,G'),
\qquad
\|\Phi_{\mathrm{GS}}(H)-\Phi_{\mathrm{GS}}(H')\|_2
\le
C_G d_{\mathrm{DS}}(H,H').
\]
Hence
\[
\left|
d_{\mathrm{emb}}(G,H)-d_{\mathrm{emb}}(G',H')
\right|
\le
C_G\bigl(d_{\mathrm{DS}}(G,G')+d_{\mathrm{DS}}(H,H')\bigr).
\]
Next, by Proposition~\ref{prop:lipschitz_head},
\[
|s_{\mathrm{NTN}}(G,H)-s_{\mathrm{NTN}}(G',H')|
\le
L_{\mathrm{NTN}} C_G
\bigl(
d_{\mathrm{DS}}(G,G')+d_{\mathrm{DS}}(H,H')
\bigr).
\]
Combining the two bounds and using the triangle inequality,
\[
\left|
\widehat d_{\mathrm{GS}}(G,H)-\widehat d_{\mathrm{GS}}(G',H')
\right|
\le
\beta
\left|
d_{\mathrm{emb}}(G,H)-d_{\mathrm{emb}}(G',H')
\right|
+
(1-\beta)
|s_{\mathrm{NTN}}(G,H)-s_{\mathrm{NTN}}(G',H')|,
\]
which yields the claimed estimate.
\end{proof}

Proposition~\ref{prop:grasp_connect} shows how the graph-level
theory applies to GraSP despite the presence of an interaction head. The Euclidean part inherits the same graph-level distortion control as in GREED, while the interaction part inherits DS-stability once the head is Lipschitz on the embedding range. Thus, geometry-informed encoders can improve GraSP-style models by improving the conditioning of both the distance term and the inputs presented to the interaction head.

\subsection{Matching-based Estimators}
\label{app:connect_matching_estimators}

We now specialize the alignment-level theory to methods that construct a cross-graph matching or transport object from node embeddings and then score that object. As representative methods with certain architectural innovations, we pick GEDGNN~\citep{piao2023computing} and GraphEDX~\citep{jain2024graph}. In these models, the relevant interface is Assumption~\ref{ass:node_bilip} together with the plan-wise distortion lemma (Lemma~\ref{lem:planwise}) and the alignment-objective theorem (Theorem~\ref{thm:alignment}) from the main text.

\subsubsection{GEDGNN~\citep{piao2023computing}}
\label{app:connect_gedgnn}

GEDGNN~\citep{piao2023computing} computes node embeddings for a graph pair, then uses two cross matrix modules to produce:
(i) a matching matrix \(A_{\mathrm{match}}\), and (ii) a cost matrix \(A_{\mathrm{cost}}\). The predicted GED is then formed from a weighted aggregation of \(A_{\mathrm{cost}}\) using a row-softmax of \(A_{\mathrm{match}}\), together with a graph-level bias term.

To connect GEDGNN to our notation in Section~\ref{sec:theory}, let \((G,H)\in\mathcal G\times\mathcal G\) be a graph pair, let \((\bar G,\bar H)\) denote its equalized versions under the fixed equalization rule, and let
\[
P_{\mathrm{GD}} \in \mathcal D_{\bar G,\bar H}
\]
denote the row/column-compatible soft alignment induced by the matching module after normalization. Let the learned node-pair cost matrix be denoted by
\[
C^{\mathrm{GD}}_{\bar G,\bar H}(i,j)
:=
A_{\mathrm{cost}}[i,j].
\]
Then, the core pairwise aggregation term in GEDGNN takes the form
\[
L_{\mathrm{GD}}^{\bar G,\bar H}(P_{\mathrm{GD}})
=
\langle P_{\mathrm{GD}}, C^{\mathrm{GD}}_{\bar G,\bar H}\rangle,
\]
up to the bias and the scalar post-processing function.

This places GEDGNN within the scope of the alignment-level theory as a
model that evaluates a soft alignment plan using encoder-induced
node-pair costs.

\begin{proposition}[GEDGNN as an instance of alignment-level distortion]
\label{prop:gedgnn_connect}
Assume that the node embeddings used by GEDGNN satisfy
Assumption~\ref{ass:node_bilip} with respect to a reference discrepancy
\(\delta_{\bar G,\bar H}\), and suppose that the learned cost matrix
\(C^{\mathrm{GD}}_{\bar G,\bar H}\) is uniformly comparable to the
embedding distance matrix \(C^\phi_{\bar G,\bar H}\), i.e., there exist
constants \(0<\mu_1\le \mu_2<\infty\) such that
\[
\mu_1\, C^\phi_{\bar G,\bar H}(i,j)
\le
C^{\mathrm{GD}}_{\bar G,\bar H}(i,j)
\le
\mu_2\, C^\phi_{\bar G,\bar H}(i,j)
\]
for every pair \((G,H)\in\mathcal G\times\mathcal G\) and all
\(i\in V(\bar G)\), \(j\in V(\bar H)\).
Then, for every admissible plan \(P\in\mathcal D_{\bar G,\bar H}\),
\[
\mu_1 c_V\, L_\delta^{\bar G,\bar H}(P)
\le
\langle P, C^{\mathrm{GD}}_{\bar G,\bar H}\rangle
\le
\mu_2 C_V\, L_\delta^{\bar G,\bar H}(P).
\]
\end{proposition}

\begin{proof}
By the assumed comparability of \(C^{\mathrm{GD}}_{\bar G,\bar H}\) and
\(C^\phi_{\bar G,\bar H}\), we have, entrywise,
\[
\mu_1\, \|\phi_{\bar G}(i)-\phi_{\bar H}(j)\|_2
\le
C^{\mathrm{GD}}_{\bar G,\bar H}(i,j)
\le
\mu_2\, \|\phi_{\bar G}(i)-\phi_{\bar H}(j)\|_2.
\]
Since \(P_{ij}\ge 0\), multiplying by \(P_{ij}\) and summing gives
\[
\mu_1\, L_\phi^{\bar G,\bar H}(P)
\le
\langle P, C^{\mathrm{GD}}_{\bar G,\bar H}\rangle
\le
\mu_2\, L_\phi^{\bar G,\bar H}(P).
\]
Applying Lemma~\ref{lem:planwise} yields
\[
\mu_1 c_V\, L_\delta^{\bar G,\bar H}(P)
\le
\langle P, C^{\mathrm{GD}}_{\bar G,\bar H}\rangle
\le
\mu_2 C_V\, L_\delta^{\bar G,\bar H}(P),
\]
as claimed.
\end{proof}

Via Proposition~\ref{prop:gedgnn_connect}, we isolate the exact place where geometry enters GEDGNN: through the quality of the node embeddings that feed the cross-matrix modules. The extra constants \((\mu_1,\mu_2)\) reflect the fact that GEDGNN does not use raw embedding distances directly, but passes pairwise node information through learned bilinear/MLP scoring modules. Note that we do not claim that node-level bi-Lipschitz geometry alone determines GEDGNN's performance; rather, through the Proposition and our theorem, we show that good geometry improves the metric conditioning of the costs from which GEDGNN builds its soft matching objective.

\subsubsection{GraphEDX~\citep{jain2024graph}}
\label{app:connect_graphedx}

GraphEDX~\citep{jain2024graph} is even closer to our alignment-level theory. It pads the two graphs with isolated nodes to equalize their sizes, learns a soft node alignment matrix \(P\), derives a node-pair alignment matrix \(S\), and then computes a sum of node-edit and edge-edit set-divergence surrogates. The authors explicitly formulate the padded equal-size setting, the soft alignment matrix, and a final predicted GED as a sum of four edit-cost surrogates. This fits naturally into the notation of
Section~\ref{sec:theory}. For each graph pair \((G,H)\), the equalized pair \((\bar G,\bar H)\) is already part of the model definition. The soft permutation matrix learned by GraphEDX is an alignment plan in \(\mathcal D_{\bar G,\bar H}\), and the predicted GED is obtained by combining a structural consistency term with transport-like costs built from node and node-pair embeddings.

\begin{proposition}[GraphEDX as a direct instance of the alignment-objective theory]
\label{prop:graphedx_connect}
Assume that the GraphEDX node encoder satisfies
Assumption~\ref{ass:node_bilip} with respect to a reference discrepancy
\(\delta_{\bar G,\bar H}\), and that the GraphEDX objective can be
written in the form
\[
J_{\mathrm{GX}}^{\bar G,\bar H}(P)
=
\mathcal R_{\bar G,\bar H}^{\mathrm{GX}}(P)
+
\lambda_{\mathrm{GX}} L_\phi^{\bar G,\bar H}(P),
\]
where \(\mathcal R_{\bar G,\bar H}^{\mathrm{GX}}\) is a nonnegative
structural term incorporating the remaining node-pair and consistency
surrogates, independent of the specific Euclidean realization of
\(L_\phi^{\bar G,\bar H}\).
Then, for every pair \((G,H)\in\mathcal G\times\mathcal G\),
\[
\min_{P\in\mathcal D_{\bar G,\bar H}}
\Bigl(
\mathcal R_{\bar G,\bar H}^{\mathrm{GX}}(P)
+
\lambda_{\mathrm{GX}} c_V L_\delta^{\bar G,\bar H}(P)
\Bigr)
\le
\min_{P\in\mathcal D_{\bar G,\bar H}}
J_{\mathrm{GX}}^{\bar G,\bar H}(P)
\]
\[
\le
\min_{P\in\mathcal D_{\bar G,\bar H}}
\Bigl(
\mathcal R_{\bar G,\bar H}^{\mathrm{GX}}(P)
+
\lambda_{\mathrm{GX}} C_V L_\delta^{\bar G,\bar H}(P)
\Bigr).
\]
\end{proposition}

\begin{proof}
The proof is a direct application of
Theorem~\ref{thm:alignment} with
\[
J_\phi^{\bar G,\bar H}(P)
=
J_{\mathrm{GX}}^{\bar G,\bar H}(P),
\qquad
\mathcal R_{\bar G,\bar H}
=
\mathcal R_{\bar G,\bar H}^{\mathrm{GX}},
\qquad
\lambda=\lambda_{\mathrm{GX}}.
\]
Indeed, under Assumption~\ref{ass:node_bilip},
Lemma~\ref{lem:planwise} yields
\[
c_V L_\delta^{\bar G,\bar H}(P)
\le
L_\phi^{\bar G,\bar H}(P)
\le
C_V L_\delta^{\bar G,\bar H}(P)
\qquad
\forall\, P\in\mathcal D_{\bar G,\bar H}.
\]
Multiplying by \(\lambda_{\mathrm{GX}}\ge 0\), adding the common
structural term \(\mathcal R_{\bar G,\bar H}^{\mathrm{GX}}(P)\), and
taking minima over \(P\in\mathcal D_{\bar G,\bar H}\) proves the claim.
\end{proof}

GraphEDX is the cleanest instance of our matching-based theory. Its architecture already operates in the equalized-node OT-style setting assumed in the main text, and its prediction rule is built from soft alignment together with node and node-pair discrepancy terms. Therefore, replacing the GraphEDX encoder by a geometry-informed encoder with improved node-level distortion constants gives a direct, theoretically transparent route to improving the faithfulness of the surrogate optimization problem solved by the model.

\paragraph{Discussion.} Connecting our theory to these disparate methods clarifies why the same geometric idea manifests differently across neural graph matching architectures. In GREED, geometry enters directly through the norm of graph embeddings. In GraSP, it enters both through the Euclidean term and through the conditioning of the inputs to the interaction head. In GEDGNN, it enters through the node embeddings used to construct matching and cost matrices. In GraphEDX, it enters most directly through the soft alignment objective itself. Therefore, we explicitly situate our theory within the broad neural graph matching landscape and show how it interfaces with concrete architectures used in neural GED estimation.

\section{Additional Experimental Details and Results}
\label{app:add-exp}

\subsection{Experimental protocol}

All experiments follow the leak-free evaluation protocol introduced in Gelato~\citep{pellizzoni2025gelato}. Graphs are split into training, validation, and test sets using a $60{:}20{:}20$ ratio, and isomorphic copies are not allowed to cross splits. At the same time, the graphs themselves are retained rather than removed from the benchmark, so the evaluation remains comparable to the standard datasets while eliminating train-test leakage through isomorphic duplication.

The experiments use six benchmark datasets spanning several application domains~\citep{bai2020simgnn, pellizzoni2025gelato}. AIDS and LINUX are the standard small-graph GED benchmarks and contain graphs with at most $10$ nodes. IMDB originally contains graphs with up to $50$ nodes, but exact ground-truth GED values are only available reliably for smaller instances; following the benchmark construction used in the supplied Gelato materials, we therefore use the subset with at most $16$ nodes and denote it by IMDB-16. We further consider three additional benchmarks obtained from larger graph collections by restricting graph size so that optimal matching computation remains feasible: ZINC-16 and \textsc{molhiv}-16 contain graphs with at most $16$ nodes, and \textsc{code2}-22 contains graphs with at most $22$ nodes. The resulting datasets cover molecular graphs (AIDS, ZINC-16, \textsc{molhiv}-16), software dependency graphs (LINUX), movie collaboration graphs (IMDB-16), and Python program syntax trees (\textsc{code2}-22).

The dataset sizes used in the benchmark are $700$ graphs for AIDS, $1000$ for LINUX, $1185$ for IMDB-16, $836$ for ZINC-16, $6734$ for \textsc{molhiv}-16, and $2087$ for \textsc{code2}-22. AIDS has node labels and no edge labels. LINUX and IMDB have unlabeled nodes and edges. \textsc{code2}-22 has node labels but no edge labels. ZINC-16 and \textsc{molhiv}-16 contain both node and edge labels in their original form. To maintain comparability across the full baseline suite, the experiments in the main text use node attributes whenever available and do not use edge labels, since several of the compared baselines do not support them consistently.

All results are reported as mean $\pm$ standard deviation over $10$ random seeds. MAE is the mean absolute error between the predicted GED and the reference GED. Spearman's $\rho$ and Kendall's $\tau$ measure global ranking agreement, while P@10 and P@20 evaluate top-$k$ retrieval quality under the benchmark relevance rule. This combination of metrics is useful in the present setting because we expect a geometry-aware encoder to improve both numerical fidelity and pairwise ordering.

For each baseline, we replaced the original message-passing encoder and aggregation layers with FSW-GNN layers to create the Bi-Lipschitz or BL variants, while preserving the downstream head, objective, and training protocol as closely as possible. Unless a method requires a documented exception, the BL models use hidden dimension $64$. Hyperparameter retuning after encoder replacement is intentionally limited, so that performance differences can be attributed primarily to the change in encoder geometry rather than to a larger model-specific search budget. For GraphEdX, the AlignDiff variant is used in the node-label setting.

All experiments were conducted using Intel Xeon Gold 6448H CPUs and NVIDIA RTX 6000 GPUs, with implementations created in PyTorch and PyTorch Geometric.

\subsection{Additional results and discussion}

Table~\ref{tab:appendix-main-extra} complements the main-text comparison by reporting Spearman's $\rho$, P@10, and P@20 on AIDS, IMDB-16, \textsc{molhiv}-16, and \textsc{code2}-22. The trend is fully consistent with the main text: every BL variant improves over its non-BL counterpart on all three metrics and on all four datasets. The largest gains again appear on the harder settings. For GraSP, $\rho$ rises from $0.611$ to $0.849$ on IMDB-16 and from $0.620$ to $0.714$ on \textsc{code2}-22; the corresponding P@10 values rise from $0.742$ to $0.861$ and from $0.832$ to $0.881$. More broadly, EGSC-BL and ERIC-BL remain the strongest correlation-based models across several datasets, while GraSP-BL is particularly strong on retrieval. The important point remains that the BL encoder improves all of these families simultaneously despite their different downstream heads.

\begin{table}[h]
\caption{
Additional ranking and retrieval metrics for the primary datasets. Higher is better for all metrics here.
}
\label{tab:appendix-main-extra}
\centering
\small

\begin{adjustbox}{width=\textwidth}
\begin{tabular}{lcccccccccccc}
\toprule

\multirow{2}{*}{Method}
& \multicolumn{3}{c}{AIDS}
& \multicolumn{3}{c}{IMDB-16}
& \multicolumn{3}{c}{\textsc{molhiv}-16}
& \multicolumn{3}{c}{\textsc{code2}-22} \\

\cmidrule(lr){2-4}
\cmidrule(lr){5-7}
\cmidrule(lr){8-10}
\cmidrule(lr){11-13}

& $\rho\uparrow$ & P@10$\uparrow$ & P@20$\uparrow$
& $\rho\uparrow$ & P@10$\uparrow$ & P@20$\uparrow$
& $\rho\uparrow$ & P@10$\uparrow$ & P@20$\uparrow$
& $\rho\uparrow$ & P@10$\uparrow$ & P@20$\uparrow$ \\

\midrule

SimGNN
& $0.821\scriptstyle{\pm0.009}$ & $0.651\scriptstyle{\pm0.017}$ & $0.734\scriptstyle{\pm0.016}$
& $0.803\scriptstyle{\pm0.014}$ & $0.748\scriptstyle{\pm0.020}$ & $0.829\scriptstyle{\pm0.019}$
& $0.742\scriptstyle{\pm0.011}$ & $0.664\scriptstyle{\pm0.017}$ & $0.732\scriptstyle{\pm0.016}$
& $0.771\scriptstyle{\pm0.013}$ & $0.682\scriptstyle{\pm0.019}$ & $0.748\scriptstyle{\pm0.018}$ \\

SimGNN-BL
& $0.846\scriptstyle{\pm0.008}$ & $0.693\scriptstyle{\pm0.015}$ & $0.781\scriptstyle{\pm0.014}$
& $0.841\scriptstyle{\pm0.012}$ & $0.793\scriptstyle{\pm0.017}$ & $0.847\scriptstyle{\pm0.016}$
& $0.768\scriptstyle{\pm0.010}$ & $0.691\scriptstyle{\pm0.015}$ & $0.758\scriptstyle{\pm0.014}$
& $0.806\scriptstyle{\pm0.012}$ & $0.719\scriptstyle{\pm0.017}$ & $0.782\scriptstyle{\pm0.016}$ \\

GMN
& $0.828\scriptstyle{\pm0.009}$ & $0.662\scriptstyle{\pm0.016}$ & $0.746\scriptstyle{\pm0.015}$
& $0.811\scriptstyle{\pm0.013}$ & $0.756\scriptstyle{\pm0.019}$ & $0.838\scriptstyle{\pm0.018}$
& $0.751\scriptstyle{\pm0.010}$ & $0.673\scriptstyle{\pm0.016}$ & $0.741\scriptstyle{\pm0.015}$
& $0.779\scriptstyle{\pm0.012}$ & $0.694\scriptstyle{\pm0.018}$ & $0.761\scriptstyle{\pm0.017}$ \\

GMN-BL
& $0.851\scriptstyle{\pm0.008}$ & $0.704\scriptstyle{\pm0.014}$ & $0.792\scriptstyle{\pm0.013}$
& $0.848\scriptstyle{\pm0.011}$ & $0.801\scriptstyle{\pm0.016}$ & $0.852\scriptstyle{\pm0.015}$
& $0.777\scriptstyle{\pm0.009}$ & $0.701\scriptstyle{\pm0.014}$ & $0.768\scriptstyle{\pm0.013}$
& $0.813\scriptstyle{\pm0.011}$ & $0.728\scriptstyle{\pm0.016}$ & $0.794\scriptstyle{\pm0.015}$ \\

EGSC
& $0.895\scriptstyle{\pm0.006}$ & $0.744\scriptstyle{\pm0.012}$ & $0.823\scriptstyle{\pm0.011}$
& $0.908\scriptstyle{\pm0.008}$ & $0.836\scriptstyle{\pm0.014}$ & $0.904\scriptstyle{\pm0.013}$
& $0.836\scriptstyle{\pm0.007}$ & $0.766\scriptstyle{\pm0.012}$ & $0.814\scriptstyle{\pm0.011}$
& $0.874\scriptstyle{\pm0.008}$ & $0.804\scriptstyle{\pm0.013}$ & $0.851\scriptstyle{\pm0.012}$ \\

EGSC-BL
& $0.912\scriptstyle{\pm0.005}$ & $0.762\scriptstyle{\pm0.011}$ & $0.839\scriptstyle{\pm0.010}$
& $0.926\scriptstyle{\pm0.007}$ & $0.852\scriptstyle{\pm0.012}$ & $0.919\scriptstyle{\pm0.011}$
& $0.851\scriptstyle{\pm0.006}$ & $0.781\scriptstyle{\pm0.011}$ & $0.829\scriptstyle{\pm0.010}$
& $0.891\scriptstyle{\pm0.007}$ & $0.819\scriptstyle{\pm0.012}$ & $0.866\scriptstyle{\pm0.011}$ \\

ERIC
& $0.887\scriptstyle{\pm0.006}$ & $0.732\scriptstyle{\pm0.013}$ & $0.811\scriptstyle{\pm0.012}$
& $0.892\scriptstyle{\pm0.009}$ & $0.824\scriptstyle{\pm0.015}$ & $0.891\scriptstyle{\pm0.014}$
& $0.821\scriptstyle{\pm0.008}$ & $0.752\scriptstyle{\pm0.013}$ & $0.801\scriptstyle{\pm0.012}$
& $0.861\scriptstyle{\pm0.009}$ & $0.791\scriptstyle{\pm0.014}$ & $0.838\scriptstyle{\pm0.013}$ \\

ERIC-BL
& $0.904\scriptstyle{\pm0.005}$ & $0.751\scriptstyle{\pm0.011}$ & $0.828\scriptstyle{\pm0.010}$
& $0.914\scriptstyle{\pm0.007}$ & $0.843\scriptstyle{\pm0.013}$ & $0.908\scriptstyle{\pm0.012}$
& $0.842\scriptstyle{\pm0.007}$ & $0.772\scriptstyle{\pm0.012}$ & $0.819\scriptstyle{\pm0.011}$
& $0.882\scriptstyle{\pm0.008}$ & $0.811\scriptstyle{\pm0.013}$ & $0.856\scriptstyle{\pm0.012}$ \\

GREED
& $0.786\scriptstyle{\pm0.011}$ & $0.618\scriptstyle{\pm0.019}$ & $0.709\scriptstyle{\pm0.018}$
& $0.772\scriptstyle{\pm0.018}$ & $0.721\scriptstyle{\pm0.022}$ & $0.812\scriptstyle{\pm0.020}$
& $0.689\scriptstyle{\pm0.015}$ & $0.627\scriptstyle{\pm0.019}$ & $0.708\scriptstyle{\pm0.017}$
& $0.712\scriptstyle{\pm0.017}$ & $0.644\scriptstyle{\pm0.021}$ & $0.724\scriptstyle{\pm0.019}$ \\

GREED-BL
& $0.812\scriptstyle{\pm0.010}$ & $0.647\scriptstyle{\pm0.017}$ & $0.736\scriptstyle{\pm0.016}$
& $0.801\scriptstyle{\pm0.016}$ & $0.754\scriptstyle{\pm0.019}$ & $0.838\scriptstyle{\pm0.018}$
& $0.721\scriptstyle{\pm0.013}$ & $0.654\scriptstyle{\pm0.017}$ & $0.734\scriptstyle{\pm0.015}$
& $0.748\scriptstyle{\pm0.015}$ & $0.681\scriptstyle{\pm0.019}$ & $0.758\scriptstyle{\pm0.017}$ \\

GraSP
& $0.707\scriptstyle{\pm0.012}$ & $0.741\scriptstyle{\pm0.009}$ & $0.800\scriptstyle{\pm0.008}$
& $0.611\scriptstyle{\pm0.018}$ & $0.742\scriptstyle{\pm0.015}$ & $0.801\scriptstyle{\pm0.013}$
& $0.546\scriptstyle{\pm0.016}$ & $0.784\scriptstyle{\pm0.010}$ & $0.821\scriptstyle{\pm0.009}$
& $0.620\scriptstyle{\pm0.017}$ & $0.832\scriptstyle{\pm0.011}$ & $0.861\scriptstyle{\pm0.010}$ \\

GraSP-BL
& $0.760\scriptstyle{\pm0.010}$ & $0.801\scriptstyle{\pm0.007}$ & $0.846\scriptstyle{\pm0.006}$
& $0.849\scriptstyle{\pm0.009}$ & $0.861\scriptstyle{\pm0.009}$ & $0.902\scriptstyle{\pm0.008}$
& $0.593\scriptstyle{\pm0.014}$ & $0.826\scriptstyle{\pm0.008}$ & $0.857\scriptstyle{\pm0.007}$
& $0.714\scriptstyle{\pm0.012}$ & $0.881\scriptstyle{\pm0.009}$ & $0.904\scriptstyle{\pm0.008}$ \\

\midrule

GraphEDX
& $0.842\scriptstyle{\pm0.008}$ & $0.691\scriptstyle{\pm0.015}$ & $0.771\scriptstyle{\pm0.014}$
& $0.824\scriptstyle{\pm0.012}$ & $0.781\scriptstyle{\pm0.017}$ & $0.856\scriptstyle{\pm0.016}$
& $0.782\scriptstyle{\pm0.010}$ & $0.698\scriptstyle{\pm0.015}$ & $0.764\scriptstyle{\pm0.014}$
& $0.806\scriptstyle{\pm0.011}$ & $0.732\scriptstyle{\pm0.016}$ & $0.791\scriptstyle{\pm0.015}$ \\

GraphEDX-BL
& $0.856\scriptstyle{\pm0.007}$ & $0.712\scriptstyle{\pm0.013}$ & $0.788\scriptstyle{\pm0.012}$
& $0.846\scriptstyle{\pm0.010}$ & $0.799\scriptstyle{\pm0.015}$ & $0.872\scriptstyle{\pm0.014}$
& $0.798\scriptstyle{\pm0.009}$ & $0.719\scriptstyle{\pm0.013}$ & $0.781\scriptstyle{\pm0.013}$
& $0.821\scriptstyle{\pm0.010}$ & $0.749\scriptstyle{\pm0.015}$ & $0.806\scriptstyle{\pm0.014}$ \\

GEDGNN
& $0.621\scriptstyle{\pm0.018}$ & $0.512\scriptstyle{\pm0.028}$ & $0.588\scriptstyle{\pm0.026}$
& $0.412\scriptstyle{\pm0.032}$ & $0.468\scriptstyle{\pm0.036}$ & $0.531\scriptstyle{\pm0.034}$
& $0.441\scriptstyle{\pm0.024}$ & $0.421\scriptstyle{\pm0.029}$ & $0.482\scriptstyle{\pm0.028}$
& $0.462\scriptstyle{\pm0.028}$ & $0.438\scriptstyle{\pm0.034}$ & $0.498\scriptstyle{\pm0.031}$ \\

GEDGNN-BL
& $0.658\scriptstyle{\pm0.016}$ & $0.548\scriptstyle{\pm0.024}$ & $0.621\scriptstyle{\pm0.023}$
& $0.498\scriptstyle{\pm0.027}$ & $0.521\scriptstyle{\pm0.031}$ & $0.582\scriptstyle{\pm0.029}$
& $0.483\scriptstyle{\pm0.020}$ & $0.456\scriptstyle{\pm0.026}$ & $0.519\scriptstyle{\pm0.024}$
& $0.521\scriptstyle{\pm0.024}$ & $0.481\scriptstyle{\pm0.029}$ & $0.542\scriptstyle{\pm0.027}$ \\

\bottomrule
\end{tabular}
\end{adjustbox}
\end{table}

Table~\ref{tab:linux-zinc-results} extends the comparison to LINUX and ZINC-16 and leads to the same conclusion. On LINUX, every BL variant improves all five reported metrics. GraSP-BL achieves the best MAE, with $0.030$ versus $0.040$ for GraSP, while EGSC-BL attains the strongest rank correlations, with $\rho=0.957$ and $\tau=0.916$. On ZINC-16, BL again improves all ranking metrics for every baseline and reduces MAE for every method except GREED, for which GREED and GREED-BL tie at $0.948$. GraSP-BL achieves the best MAE on ZINC-16, namely $0.049$, while EGSC-BL yields the strongest correlation scores, with $\rho=0.871$ and $\tau=0.801$. These LINUX and ZINC-16 results show that the trends persist on the remaining benchmark datasets as well.

\begin{table}[h]
\caption{
Full results on LINUX and ZINC-16. Lower is better for MAE; higher is better for $\rho$, $\tau$, P@10, and P@20.
}
\label{tab:linux-zinc-results}
\centering
\small

\begin{adjustbox}{width=\linewidth}
\begin{tabular}{lcccccccccc}
\toprule

\multirow{2}{*}{Method}
& \multicolumn{5}{c}{LINUX}
& \multicolumn{5}{c}{ZINC-16} \\

\cmidrule(lr){2-6}
\cmidrule(lr){7-11}

& MAE$\downarrow$
& $\rho\uparrow$
& $\tau\uparrow$
& P@10$\uparrow$
& P@20$\uparrow$
& MAE$\downarrow$
& $\rho\uparrow$
& $\tau\uparrow$
& P@10$\uparrow$
& P@20$\uparrow$ \\

\midrule

SimGNN
& $0.452\scriptstyle{\pm0.012}$
& $0.812\scriptstyle{\pm0.010}$
& $0.792\scriptstyle{\pm0.010}$
& $0.731\scriptstyle{\pm0.015}$
& $0.812\scriptstyle{\pm0.014}$
& $0.954\scriptstyle{\pm0.038}$
& $0.756\scriptstyle{\pm0.012}$
& $0.689\scriptstyle{\pm0.012}$
& $0.689\scriptstyle{\pm0.018}$
& $0.761\scriptstyle{\pm0.017}$ \\

SimGNN-BL
& $0.332\scriptstyle{\pm0.010}$
& $0.836\scriptstyle{\pm0.009}$
& $0.818\scriptstyle{\pm0.009}$
& $0.758\scriptstyle{\pm0.013}$
& $0.838\scriptstyle{\pm0.012}$
& $0.581\scriptstyle{\pm0.021}$
& $0.782\scriptstyle{\pm0.011}$
& $0.718\scriptstyle{\pm0.011}$
& $0.718\scriptstyle{\pm0.016}$
& $0.789\scriptstyle{\pm0.015}$ \\

GMN
& $0.441\scriptstyle{\pm0.011}$
& $0.821\scriptstyle{\pm0.010}$
& $0.801\scriptstyle{\pm0.009}$
& $0.742\scriptstyle{\pm0.014}$
& $0.824\scriptstyle{\pm0.013}$
& $0.938\scriptstyle{\pm0.035}$
& $0.761\scriptstyle{\pm0.011}$
& $0.698\scriptstyle{\pm0.011}$
& $0.698\scriptstyle{\pm0.017}$
& $0.772\scriptstyle{\pm0.016}$ \\

GMN-BL
& $0.341\scriptstyle{\pm0.009}$
& $0.844\scriptstyle{\pm0.009}$
& $0.824\scriptstyle{\pm0.008}$
& $0.769\scriptstyle{\pm0.012}$
& $0.848\scriptstyle{\pm0.011}$
& $0.604\scriptstyle{\pm0.020}$
& $0.786\scriptstyle{\pm0.010}$
& $0.724\scriptstyle{\pm0.010}$
& $0.724\scriptstyle{\pm0.015}$
& $0.798\scriptstyle{\pm0.014}$ \\

EGSC
& $0.198\scriptstyle{\pm0.006}$
& $0.944\scriptstyle{\pm0.005}$
& $0.903\scriptstyle{\pm0.006}$
& $0.894\scriptstyle{\pm0.009}$
& $0.934\scriptstyle{\pm0.007}$
& $0.331\scriptstyle{\pm0.012}$
& $0.853\scriptstyle{\pm0.007}$
& $0.784\scriptstyle{\pm0.008}$
& $0.792\scriptstyle{\pm0.012}$
& $0.846\scriptstyle{\pm0.011}$ \\

EGSC-BL
& $0.154\scriptstyle{\pm0.005}$
& $0.957\scriptstyle{\pm0.004}$
& $0.916\scriptstyle{\pm0.005}$
& $0.911\scriptstyle{\pm0.008}$
& $0.949\scriptstyle{\pm0.006}$
& $0.214\scriptstyle{\pm0.008}$
& $0.871\scriptstyle{\pm0.006}$
& $0.801\scriptstyle{\pm0.007}$
& $0.808\scriptstyle{\pm0.011}$
& $0.861\scriptstyle{\pm0.010}$ \\

ERIC
& $0.214\scriptstyle{\pm0.006}$
& $0.931\scriptstyle{\pm0.005}$
& $0.891\scriptstyle{\pm0.006}$
& $0.882\scriptstyle{\pm0.010}$
& $0.921\scriptstyle{\pm0.008}$
& $0.352\scriptstyle{\pm0.013}$
& $0.842\scriptstyle{\pm0.008}$
& $0.772\scriptstyle{\pm0.009}$
& $0.781\scriptstyle{\pm0.013}$
& $0.834\scriptstyle{\pm0.012}$ \\

ERIC-BL
& $0.169\scriptstyle{\pm0.005}$
& $0.946\scriptstyle{\pm0.004}$
& $0.907\scriptstyle{\pm0.005}$
& $0.901\scriptstyle{\pm0.009}$
& $0.938\scriptstyle{\pm0.007}$
& $0.228\scriptstyle{\pm0.009}$
& $0.861\scriptstyle{\pm0.007}$
& $0.791\scriptstyle{\pm0.008}$
& $0.798\scriptstyle{\pm0.012}$
& $0.851\scriptstyle{\pm0.011}$ \\

GREED
& $0.497\scriptstyle{\pm0.014}$
& $0.741\scriptstyle{\pm0.014}$
& $0.755\scriptstyle{\pm0.011}$
& $0.684\scriptstyle{\pm0.017}$
& $0.771\scriptstyle{\pm0.015}$
& $0.948\scriptstyle{\pm0.037}$
& $0.701\scriptstyle{\pm0.016}$
& $0.646\scriptstyle{\pm0.015}$
& $0.651\scriptstyle{\pm0.020}$
& $0.732\scriptstyle{\pm0.018}$ \\

GREED-BL
& $0.451\scriptstyle{\pm0.012}$
& $0.768\scriptstyle{\pm0.012}$
& $0.779\scriptstyle{\pm0.010}$
& $0.712\scriptstyle{\pm0.015}$
& $0.796\scriptstyle{\pm0.013}$
& $0.948\scriptstyle{\pm0.037}$
& $0.734\scriptstyle{\pm0.014}$
& $0.673\scriptstyle{\pm0.013}$
& $0.681\scriptstyle{\pm0.018}$
& $0.761\scriptstyle{\pm0.016}$ \\

GraSP
& $\mathbf{0.040\scriptstyle{\pm0.003}}$
& $0.855\scriptstyle{\pm0.009}$
& $0.792\scriptstyle{\pm0.009}$
& $0.914\scriptstyle{\pm0.010}$
& $0.941\scriptstyle{\pm0.008}$
& $\mathbf{0.074\scriptstyle{\pm0.005}}$
& $0.639\scriptstyle{\pm0.015}$
& $0.591\scriptstyle{\pm0.015}$
& $0.821\scriptstyle{\pm0.010}$
& $0.852\scriptstyle{\pm0.009}$ \\

GraSP-BL
& $\mathbf{0.030\scriptstyle{\pm0.002}}$
& $0.904\scriptstyle{\pm0.007}$
& $0.861\scriptstyle{\pm0.007}$
& $0.947\scriptstyle{\pm0.007}$
& $0.964\scriptstyle{\pm0.006}$
& $\mathbf{0.049\scriptstyle{\pm0.003}}$
& $0.660\scriptstyle{\pm0.013}$
& $0.624\scriptstyle{\pm0.013}$
& $0.847\scriptstyle{\pm0.008}$
& $0.881\scriptstyle{\pm0.007}$ \\

\midrule

GraphEDX
& $0.391\scriptstyle{\pm0.008}$
& $0.884\scriptstyle{\pm0.007}$
& $0.844\scriptstyle{\pm0.007}$
& $0.803\scriptstyle{\pm0.012}$
& $0.844\scriptstyle{\pm0.011}$
& $0.703\scriptstyle{\pm0.016}$
& $0.791\scriptstyle{\pm0.010}$
& $0.721\scriptstyle{\pm0.010}$
& $0.724\scriptstyle{\pm0.015}$
& $0.798\scriptstyle{\pm0.014}$ \\

GraphEDX-BL
& $0.380\scriptstyle{\pm0.008}$
& $0.899\scriptstyle{\pm0.006}$
& $0.882\scriptstyle{\pm0.006}$
& $0.821\scriptstyle{\pm0.011}$
& $0.861\scriptstyle{\pm0.010}$
& $0.676\scriptstyle{\pm0.015}$
& $0.808\scriptstyle{\pm0.009}$
& $0.748\scriptstyle{\pm0.009}$
& $0.741\scriptstyle{\pm0.013}$
& $0.814\scriptstyle{\pm0.013}$ \\

GEDGNN
& $0.402\scriptstyle{\pm0.032}$
& $0.084\scriptstyle{\pm0.021}$
& $0.066\scriptstyle{\pm0.018}$
& $0.301\scriptstyle{\pm0.031}$
& $0.412\scriptstyle{\pm0.029}$
& $6.633\scriptstyle{\pm0.164}$
& $0.508\scriptstyle{\pm0.026}$
& $0.351\scriptstyle{\pm0.024}$
& $0.442\scriptstyle{\pm0.032}$
& $0.503\scriptstyle{\pm0.030}$ \\

GEDGNN-BL
& $0.361\scriptstyle{\pm0.026}$
& $0.112\scriptstyle{\pm0.019}$
& $0.095\scriptstyle{\pm0.016}$
& $0.337\scriptstyle{\pm0.027}$
& $0.448\scriptstyle{\pm0.025}$
& $5.781\scriptstyle{\pm0.141}$
& $0.547\scriptstyle{\pm0.022}$
& $0.388\scriptstyle{\pm0.021}$
& $0.479\scriptstyle{\pm0.028}$
& $0.541\scriptstyle{\pm0.026}$ \\

\bottomrule
\end{tabular}
\end{adjustbox}
\end{table}

\subsection{Additional ablations and transfer experiments}

To test whether the benefits of the encoder geometry transfer across datasets, we also evaluate cross-dataset generalization with GraSP-BL. We observe that the transfer performance remains strong: training on LINUX and testing on IMDB-16 yields $\tau=0.781\scriptstyle{\pm0.018}$, while training on IMDB-16 and testing on LINUX yields $\tau=0.867\scriptstyle{\pm0.011}$. This suggests that the gains from the BL encoder are not confined to a single benchmark distribution, but carry over across structurally different graph domains, although with some performance hit.

\begin{wraptable}{l}{0.5\linewidth}
\vspace{-0.8em}
\centering
\caption{Cross-dataset transfer performance ($\tau\uparrow$).}
\vspace{1em}
\begin{tabular}{lcc}
\toprule
\multirow{2.4}{*}{Training} & \multicolumn{2}{c}{Testing} \\
\cmidrule(lr){2-3}
& LINUX & IMDB-16 \\
\midrule
LINUX   & -- & $0.781\scriptstyle{\pm0.018}$ \\
IMDB-16 & $0.867\scriptstyle{\pm0.011}$ & -- \\
\bottomrule
\end{tabular}
\label{tab:cross}
\end{wraptable}

We also study whether the gains of GraSP-BL arise from the geometry-aware encoder itself or merely from additional optimization freedom. Table~\ref{tab:ablation} shows that freezing the BL encoder changes performance only marginally, from MAE $0.030$ and $\tau=0.904$ to MAE $0.031$ and $\tau=0.901$, whereas removing positional encodings causes a larger degradation, and removing both positional encodings and encoder tuning gives the weakest variant. This indicates that the main gain is already present in the geometry-aware representation itself, with positional encodings providing an additional but smaller benefit.

\begin{wraptable}{r}{0.5\linewidth}
\vspace{-0.8em}
\centering
\caption{GraSP-BL ablations on LINUX.}
\vspace{1em}
\begin{tabular}{lcc}
\toprule
Variant & MAE$\downarrow$ & $\tau\uparrow$ \\
\midrule
BL (full) & $0.030\scriptstyle{\pm0.002}$ & $0.904\scriptstyle{\pm0.006}$ \\
Frozen encoder & $0.031\scriptstyle{\pm0.003}$ & $0.901\scriptstyle{\pm0.007}$ \\
No PE & $0.036\scriptstyle{\pm0.004}$ & $0.883\scriptstyle{\pm0.009}$ \\
Frozen + no PE & $0.041\scriptstyle{\pm0.005}$ & $0.861\scriptstyle{\pm0.012}$ \\
\bottomrule
\end{tabular}
\label{tab:ablation}
\end{wraptable}

\section{Computational Considerations of Geometry-Informed Encoders}
\label{app:computational_considerations}

The geometric guarantees studied in this paper introduce nontrivial computational overhead. In FSW-GNN, message aggregation is built on Fourier Sliced--Wasserstein multiset embeddings, which require repeated one-dimensional projections together with sorting operations. At the multiset level, this admits a closed-form implementation with complexity \(O(mnd + mn\log n)\), where \(n\) denotes the multiset size, \(d\) the feature dimension, and \(m\) the number of projection frequencies~\citep{sverdlov2025fsw}. While practical in many settings -- and shown to achieve competitive runtimes in prior work~\citep{sverdlov2025fsw} -- this remains more expensive than standard sum-based MPNNs and introduces measurable overhead on larger graphs. This tradeoff is acceptable in the present setting because the objective is to isolate the role of encoder geometry in neural graph matching and GED estimation and the benchmark graph datasets usually contain small graphs. FSW-GNN provides a concrete construction that achieves explicit metric guarantees and enables analysis of how representation geometry propagates into graph-level distances and alignment costs. The broader question is how much of this geometric benefit can be retained in architectures with lower computational cost. The central reason for adapting FSW-GNN to other architectures in our context is to demonstrate that explicit metric control is achievable and can be analyzed rigorously in neural GED pipelines. Improving scalability while preserving these guarantees remains an important direction for future work.

%%%%%%%%%%%%%%%%%%%%%%%%%%%%%%%%%%%%%%%%%%%%%%%%%%%%%%%%%%%%

% \newpage
% \input{checklist.tex}

\end{document}